\documentclass[mnsc,nonblindrev]{informs3noheader} 
\OneAndAHalfSpacedXI 
\usepackage{amsmath, amssymb, mathtools, bm}
\usepackage{booktabs}
\usepackage{enumitem}
\usepackage{hyperref}
\usepackage{cleveref}
\usepackage{comment}
\usepackage{caption}
\usepackage{subcaption}
\usepackage{float}

\newcommand{\E}{\mathbb{E}}

\usepackage{bbm}
\newcommand*{\horzbar}{\rule[.5ex]{2.5ex}{0.5pt}}

\usepackage{natbib}
 \bibpunct[, ]{(}{)}{,}{a}{}{,}%
 \def\bibfont{\small}%

\TheoremsNumberedThrough    
\ECRepeatTheorems
\EquationsNumberedThrough  
\MANUSCRIPTNO{}

\begin{document}

\RUNAUTHOR{Bastani et al.}
\RUNTITLE{False Promises in Data-Driven Decisions}

\TITLE{Winner's Curse Drives False Promises\\ in Data-Driven Decisions}

\ARTICLEAUTHORS{%
\AUTHOR{Hamsa Bastani}
\AFF{Wharton School, \EMAIL{hamsab@wharton.upenn.edu}} 
\AUTHOR{Osbert Bastani}
\AFF{University of Pennsylvania, \EMAIL{obastani@seas.upenn.edu}}
\AUTHOR{Bryce McLaughlin}
\AFF{Carroll School of Management, \EMAIL{bryce.mclaughlin@bc.edu}}
} 

\ABSTRACT{A major challenge in data-driven decision-making is accurate policy evaluation---i.e., guaranteeing that a learned decision-making policy achieves the promised benefits. A popular strategy is \emph{model-based} policy evaluation, which estimates a model from data to infer counterfactual outcomes. This strategy is known to produce unwarrantedly optimistic estimates of the true benefit due to the \emph{winner's curse}. We searched the recent literature on data-driven decision-making, identifying a sample of 55 papers published in \textit{Management Science} in the past decade; all but two relied on this flawed methodology. Several common justifications are provided: (1) the estimated models are accurate, stable, and well-calibrated, (2) the historical data uses random treatment assignment, (3) the model family is well-specified, and (4) the evaluation methodology uses sample splitting. Unfortunately, we show that no combination of these justifications avoids the winner's curse. First, we provide a theoretical analysis demonstrating that the winner's curse can cause large, spurious reported benefits even when all these justifications hold. Second, we perform a simulation study based on the recent and consequential data-driven refugee matching problem. We construct a synthetic refugee matching environment (calibrated to closely match the real setting) but designed so that no assignment policy can improve expected employment compared to random assignment. Model-based methods report large, stable gains of around 60\% even when the true effect is zero; these gains are on par with improvements of 22--75\% reported in the literature. Our results provide strong evidence against model-based evaluation.
}

\maketitle

\section{Introduction}
\label{sec:intro}

The past two decades have seen tremendous gains in the accuracy of statistical and predictive models, due to a confluence of data availability, computational resources, and algorithmic advances. A major application of such estimation is for data-driven decision-making, typically focusing on high-dimensional decision spaces---e.g., targeting interventions to patients, discounts to customers, or matching refugees to cities based on observed individual covariates. In practice, a decision-maker gathers a historical dataset, estimates a model of both outcomes and treatment effects, and then uses this model to learn a good decision-making policy. 

Naturally, one would want to know if and by how much this learned data-driven policy improves the decision-maker's objective (e.g., revenue or social welfare) over the existing policy, which we call \emph{policy improvement}. This requires \textit{policy evaluation}, which estimates the purported improvement of the data-driven policy over the existing policy. A variety of methods have been proposed for estimating policy improvement with varying assumptions and statistical guarantees. We posit that a key goal of policy evaluation is to guarantee that the true policy improvement lies within a reported confidence interval with high probability---i.e., a decision-maker can have confidence that the proposed policy has benefits consistent with what was promised. We say a policy evaluation methodology is \emph{valid} if it provides this guarantee.

While this sounds straightforward, the key challenge is that we do not observe counterfactual outcomes for decisions that differ from the ones in the historical data. As a consequence, we cannot directly estimate the efficacy of a decision-making policy on a held-out test set the way we can estimate the accuracy of a prediction model.

When the historical dataset uses random treatment assignment (e.g., the data was collected via a randomized controlled trial (RCT)), then the gold standard policy evaluation methodology is inverse probability weighting (IPW)~\citep{horvitz1952generalization, rosenbaum_central_1983} on a held-out test set. Intuitively, IPW only evaluates the chosen decisions when they align with those in the historical data, altogether avoiding the need to estimate counterfactual outcomes. We call this a \textit{model-free} approach. These estimates can be combined with standard statistical methodologies to establish valid confidence intervals for policy improvement. However, there are several reasons why IPW may not be feasible. When the action space is large or the covariates are high-dimensional, IPW can have very high variance; yet, these are often the settings where data-driven decisions are most desirable. Alternatively, the treatment assignment may not be random, in which case IPW is invalid.\footnote{We are also assuming the treatment probabilities are known; they can alternatively be estimated from data, but this strategy relies on the assumption that there are no unobserved confounders.}

In these cases, practitioners often adopt a popular alternative where they use a predictive model to estimate counterfactual outcomes, and then use these counterfactual outcomes to estimate policy improvement. We call this a \textit{model-based} approach. The model can range from a simple linear model to a dynamic Markov decision process to a modern machine learning model such as a random forest. The key feature is that they estimate policy improvement by using their estimated model to impute unobserved outcomes under counterfactual treatment decisions.

It is well known that using the same predictive model for both policy optimization and policy evaluation results in an issue known as the \emph{winner's curse}~\citep{harrison1984decision,smith2006optimizer}, where the estimated policy improvements are optimistically biased. Intuitively, the optimization procedure exploits errors in the estimated counterfactuals, choosing decisions whose predicted quality exceeds their true quality. If we use the same predictive model to evaluate these decisions, then these estimation errors go uncorrected, leading to optimistic bias. This unwarranted optimistic bias makes model-based evaluation methods invalid, potentially leading them to predict substantial policy improvements where none exist.

While the winner's curse is well-known, existing work has only shown it to hold under specific conditions---specifically, the well-specified setting where estimation errors are purely due to random noise in the outcomes. As a consequence, many papers provide seemingly plausible justifications and corrections for using model-based evaluation---e.g., using sample splitting to ``decouple'' the noise exploited by the optimization procedure from the noise encountered during policy evaluation. To understand the nature of these justifications, we performed a survey of recent \textit{Management Science} papers on data-driven decision-making (detailed in Section~\ref{sec:survey}). Of the 55 papers that report on data-driven policy improvement, all but two use the model-based (rather than model-free) method. Surprisingly, a large number of papers provided no justification. When justifications were provided, they typically fell into one of the following categories:
\begin{itemize}
\item \textbf{J1: Accurate, stable, and calibrated:} The estimated model achieves high accuracy on a held-out test set, is stable across data subsamples (e.g., bootstrap), and is well-calibrated.
\item \textbf{J2: Random treatment assignment:} The historical data has known random assignment.
\item \textbf{J3: Well-specified:} The model family is sufficiently rich to include the true data-generating process (e.g., random forests).
\item \textbf{J4: Sample splitting:} The evaluation splits the historical data into two disjoint samples and trains a different model on each sample, one for policy learning and one for policy evaluation.
\end{itemize}
To see why these justifications are plausible, consider the following argument: (1) if either J1 and J2 hold or J3 holds, then the estimated counterfactual outcomes are unbiased or have small bias; (2) if, in addition, J4 holds, then the estimated policy improvement is independent from the learned policy; (3) therefore, the estimated policy improvement is accurate. Indeed, these justifications would be sufficient to mitigate the version of the winner's curse described in existing work~\citep{harrison1984decision,smith2006optimizer}.

Unfortunately, we show both theoretically and through a simulation study that this argument is incorrect. First, we provide two stylized examples showing that the winner's curse can be arbitrarily large while simultaneously satisfying all of these justifications; intuitively, our examples exploit biases in model predictions due to misspecification and regularization, respectively, to produce optimistically biased policy improvement estimates. Second, we perform a realistic simulation experiment in the context of the refugee matching problem~\citep{bansak2018improving,ahani2021placement} to show that model-based evaluation can report very large spurious effects even when the reported justifications are met. Our results show that the winner's curse is substantially more pervasive than previously understood. They provide strong evidence against the validity of existing model-based policy evaluation methods, suggesting that they should not be used.

We highlight the critical need for better policy evaluation methodologies that are valid while addressing the shortcomings of IPW. For instance, recent work has studied promising directions that integrate policy learning and evaluation to provide lower-variance policy improvement estimates while preserving validity---e.g., targeting statistically significant policy improvement instead of purely maximizing expected outcomes~\citep{bastani2025beating,chernozhukov2025policy} or reducing the dimensionality of the policy class to mitigate the winner's curse~\citep{banerjee2025selecting}.

\subsection{Our Contributions}

\paragraph{Theoretical analysis.}
Intuitively, the winner's curse does not arise if the estimated counterfactuals are unbiased and sample splitting is used. However, unbiased estimates are a very strong assumption. In practice, one of two possible shortcomings can arise:
\begin{itemize}
\item \textbf{Misspecification:} A simple parametric model (e.g., linear) is unlikely to match the true data-generating process, even if it achieves a highly accurate fit on the training set distribution. In this case, even with an infinitely large, high-quality dataset, the estimated model is biased. We provide a simple example showing that misspecification can lead to arbitrarily large, unsubstantiated estimated policy improvements even when justifications J1, J2, and J4 hold. This analysis is detailed in Section~\ref{sec:simple_illustration}.
\item \textbf{Regularized:} One way to avoid misspecification is to use modern non-parametric algorithms such as random forests or gradient boosted machines (GBMs). However, these models must be regularized to avoid overfitting. We show that even when the model family is well-specified, the bias induced by regularization in finite sample can cause the winner's curse. While this bias is small and stable on the historical data distribution from which the training data is sampled, it can become significant on the shifted data distribution induced by a different policy. We provide an explicit construction via a toy ridge regression model showing it can lead to arbitrarily large, unsubstantiated estimated policy improvements even when all the above justifications (J1, J2, J3, and J4) hold. This analysis is detailed in Section~\ref{sec:regularized}.
\end{itemize}
The literature specifically on the winner's curse typically focuses on the well-specified setting without sample splitting~\citep{harrison1984decision,smith2006optimizer}; furthermore, they generally rely on the estimated model being noisy (i.e., achieves low accuracy on a held-out test set) and unstable (i.e., parameter estimates can change drastically when the model is trained on new or bootstrapped samples). Our analysis shows that due to regularization, the winner's curse can occur under much more general conditions. Separately, there has been work showing that misspecification can result in bias~\citep{jiang2016doubly, kallus2019intrinsically}; however, they do not show that this bias can be systematically positive, or that it can arise even when the estimated model is highly accurate and stable.

\paragraph{Simulation study.}

Empirically, we perform a simulation experiment in the context of data-driven refugee matching~\citep{bansak2018improving,ahani2021placement}, where the goal is to match refugees to capacity-constrained cities to optimize overall employment outcomes. While model-based evaluations are widely used in the literature, we focus on refugee matching both due to widespread recent interest in this problem as well as the fact that it has been deployed across several countries to make placement decisions for thousands of refugees. Both~\citet{bansak2018improving} and~\citet{ahani2021placement} use the model-based method to evaluate proposed data-driven refugee matching algorithms, estimating 22--75\% gains in refugee employment outcomes from data-driven placement decisions. We design a simulated dataset that closely matches the setting described in~\cite{bansak2018improving}, but constructed so that no policy can outperform random assignment. Thus, any estimated performance gains are spurious by design. Applying the algorithm and model-based evaluation method from~\cite{bansak2018improving} to this dataset, we obtain (spurious) estimated gains in employment outcomes of around 60\%, which is on par with the estimates reported in their paper. 
We additionally consider a more sophisticated variation of the model-based evaluation method using bootstrapping~\citep{ahani2021placement},
and find that it also continues to produce substantial, stable estimated improvements even when the ground truth is null.\footnote{To the best of our knowledge, all published estimates for this problem use some variation of the model-based method. Prior to posting this paper, we shared our critique of model-based evaluation methods in private correspondence with~\cite{bansak2018improving} and~\cite{ahani2021placement}; in response,~\citet{bansak_robustness_2026} developed model-free estimates to rigorously establish performance gains.}

\subsection{Related Work}
\label{sec:rel_work}

Policy evaluation methods aim to evaluate a candidate treatment assignment policy based on historical data (called the \emph{historical dataset}) collected using a different policy (called the \emph{historical policy}). They can be categorized into model-based approaches (which train a predictive model to estimate counterfactual outcomes) and model-free approaches (which do not).
We overview existing work on policy evaluation, with an emphasis on the assumptions that are required for them to provide accurate estimates. We organize these approaches into four categories:
\begin{itemize}
\item \textbf{Model-free evaluation:} These techniques do not require training a predictive model to estimate counterfactuals; they are based on inverse propensity weighting (IPW)~\citep{horvitz1952generalization,rosenbaum_central_1983}. The only require that the propensities are known---e.g., the historical data comes from a randomized controlled trial (RCT).
\item \textbf{Model-based evaluation:} These techniques train a predictive model to estimate counterfactuals; they suffer from optimistic bias due to the winner's curse.
\item \textbf{Corrections for model-based evaluation:} These techniques aim to correct for the winner's curse in model-based evaluation. They largely focus on the setting where the estimated counterfactuals are unbiased, and the winner's curse arises purely from reusing the same dataset for policy learning and policy evaluation. In this setting, sample splitting also corrects for the winner's curse; these corrections aim to improve sample efficiency over sample splitting.
\item \textbf{Policy learning accounting for policy evaluation:} A different but complementary line of work seeks to design policy learning algorithms that maximize the chances of statistically significant policy evaluation. For instance, even when the propensities are known, some work avoids IPW due to its high variance. These approaches mitigate this issue by learning policies that maximize the probability of statistically significant policy improvements according to IPW (instead of maximizing expected performance).
\end{itemize}
We summarize related work in each of these categories below.

\paragraph{Model-free evaluation.}

Model-free approaches are largely based on inverse propensity weighting (IPW)~\citep{horvitz1952generalization,rosenbaum_central_1983}. Assuming (1) knowledge of the historical policy (specifically, its treatment assignment probabilities),\footnote{Note that model-free methods implicitly require the historical policy to be stochastic or they cannot be applied.} and (2) that the candidate policy is statistically independent of the historical dataset, IPW is guaranteed to produce an unbiased estimate of the performance of the candidate policy.

When the historical dataset comes from a randomized controlled trial (RCT) where treatments are selected uniformly at random, then the first assumption is automatically satisfied. In settings where the candidate policy is selected using a policy optimization algorithm, then the second assumption can be satisfied by using sample splitting---i.e., optimize the candidate policy on a training set and evaluate it on a held-out test set. Then, the assumption holds since the test set is independent of the training set (and, therefore, the candidate policy).

These techniques are very general---e.g., they can be applied to continuous treatments~\citep{xie_towards_2019}
or Markov decision processes~\citep{nachum_dualdice:_2019}. In addition, IPW can be combined with model-based counterfactual estimates, resulting in the augmented IPW (AIPW) method~\citep{robins1994estimation} (also called doubly robust evaluation~\citep{dudik}). AIPW retains the guarantees of IPW while reducing variance provided the counterfactual estimates are sufficiently accurate.

\paragraph{Model-based evaluation.}

Model-free evaluation can be undesirable for two reasons---the historical policy is not known, or model-free approaches have high variance (especially problematic for large treatment spaces). In these settings, many papers fall back on model-based evaluation methods (also called \textit{direct methods}~\cite{dudik}). However, model-based methods require substantially stronger assumptions to provide unbiased performance estimates---specifically, they require (1) the estimated counterfactuals are unbiased, and (2) the candidate policy is statistically independent of both the predictive model and the historical dataset. In this case, (2) can be guaranteed by optimizing the candidate policy on a training set, training the predictive model on a second training set, and then evaluating the policy on a third held-out test set.

However, (1) is a very strong assumption. On the one hand, using a more traditional statistical technique such as ordinary least squares ensures that the parameter estimates are an unbiased estimate of the optimal parameters (i.e., the parameters obtained in the limit of infinite data), but the model is likely to be misspecified so even the optimal parameters produce biased counterfactual estimates. Our example in Section~\ref{sec:simple_illustration} is a consequence of this issue.

On the other hand, using non-parametric models such as random forests or gradient boosted models can fit arbitrarily complex distributions given sufficiently large datasets; as such, recent work has proposed using them to estimate counterfactuals~\citep{wager2018estimation,athey_generalized_2019}. However, they are heavily regularized to prevent overfitting. Regularization can produce systematic bias in the counterfactual estimates---intuitively, regularization pushes estimates towards zero in low-data regions where the model is highly uncertain. As shown in our example in Section~\ref{sec:regularized}, this kind of systematic bias can be exploited by policy optimization to predict spurious improvements.

\paragraph{Corrections for model-based evaluation.}

A number of approaches have been proposed for debiasing model-based evaluation to account for the winner's curse. These corrections largely focus on the traditional source of the winner's curse---e.g., assuming that the counterfactual estimates themselves are unbiased, and the winner's curse arises purely from dependence between the candidate policy and the predictive model. As described above, in this setting, sample splitting is sufficient to avoid the winner's curse. However, it can be very sample inefficient---the most effective sample splitting strategy requires four separate samples: one for training a predictive model for policy optimization, one for performing policy optimization, one for training a predictive model for policy evaluation, and one for performing policy evaluation. In this framing, in-sample corrections for model-based evaluation are designed to avoid the cost of sample splitting.

The simplest approach to improving sample complexity is to use bootstrapped samples instead of sample splitting~\citep{xu2025winner,karim_does_2026}. However, bootstrapping still reduces effective sample size, motivating alternative approaches that propose analytical corrections that retain use of the full sample. One approach is to apply techniques from post-selection inference to estimate the improvement in performance of the candidate policy conditioned on the event that it was selected by the policy optimization procedure~\citep{andrews_inference_2024,petrou-zeniou_inference_2024,zrnic_flexible_2024}. 
Along the same lines, optimization-aware debiasing techniques estimate and remove the winner's curse by examining how model estimation errors interact with the policy optimization procedure~\citep{gupta2024debiasing,iyengar_optimizers_2025,guo_off-policy_2022}. In general, analytical approaches provide better sample efficiency but must be tailored to a specific policy optimization method.

Broadly speaking, all of these corrections make some kind of assumption that the estimated counterfactuals are unbiased. This assumption always holds in the special case where we are considering applying the same treatments for all individuals (e.g., A/B testing) and the number of treatment options is finite. In this case, we are simply estimating the population mean of the outcome under each treatment, and the standard mean estimator is unbiased. Several of the corrections discussed above focus on this setting, ensuring their validity. Our work demonstrates that these approaches cannot be extended to targeting treatments based on individual covariates.

\paragraph{Policy learning accounting for policy evaluation.}

So far, we have focused on policy evaluation methods designed to correct for the winner's curse. An emerging line of work proposes policy optimization algorithms that output policies designed to be more amenable to policy evaluation. 
For instance, one of the shortcomings of IPW is its high variance; as a consequence, even if it estimates a large mean policy improvement, the confidence interval around this estimate may still include zero. Thus, recent work has proposed policy learning algorithms that seek to produce policies with confident policy improvements instead of maximizing expected policy improvement~\citep{swaminathan2015batch,chernozhukov2025policy,bastani2025beating}. In the same spirit, there have been proposed approaches that maximize a conservative lower bound on policy improvement instead of the mean policy improvement~\citep{kumar2020conservative,boloori_data-driven_2020,van_parys_data_2021,kallus_assessing_2022} or that regularize the shift between the learned policy and the historical policy (e.g., using KL divergence or Wasserstein distance)~\citep{nachum_dualdice:_2019,yao_generative_2024}. An alternative approach is to reduce the variance of IPW estimation by simplifying the problem---e.g., by grouping similar treatments~\citep{saito2022off,ladhania_personalized_2023,spiess_finding_2023,banerjee2025selecting}. All of these techniques are complementary, improving the chances of confident policy improvement when using valid policy evaluation methods.

\section{A Survey of the Literature} \label{sec:survey}

To understand the extent to which model-based policy optimization is currently used in the literature---and consequently, the prevalence of the winner's curse---we surveyed recent papers published in \textit{Management Science}. We focused on this journal as a leading outlet for quantitative research at the intersection of data analysis (estimation) and decision-making (optimization).

We constructed a sample of recent papers using the EBSCO Business Complete database. We searched for papers published in the last ten years (2015--2024) that contained keywords related to both data-driven estimation and policy optimization. Specifically, we used the query \texttt{(data OR estim*) AND (optim* OR policy)} and restricted the search to peer-reviewed articles in \textit{Management Science}. This initial search yielded 875 results.
To efficiently screen this large corpus, we employed a two-stage classification procedure involving a Large Language Model (LLM) followed by manual verification. In the first stage, we fed the titles and abstracts of the 875 papers to the GPT-5.2 API with high reasoning. We prompted the model to identify papers that likely met three specific criteria: (1) one of the primary contributions involves estimating a model (e.g., demand, welfare) from data; (2) the estimated model is used to optimize a decision (e.g., pricing, routing, allocation); and (3) the abstract explicitly states an estimated improvement resulting from this optimized decision, suggesting that this performance gain is a headline contribution. To calibrate the model, we provided 39 manually labeled examples (both positive and negative) and instructed it to follow the provided reasoning. This automated screening process identified 111 papers as potentially fitting our criteria (classified as ``yes'' or ``unclear'').

In the second stage, we manually downloaded and reviewed the full text of these 111 papers. We retained only those fitting the ``estimate-then-optimize'' framework where evaluation was performed on historical data with unknown counterfactuals. We further excluded papers where the optimization was purely theoretical without empirical estimation, or where the estimation was not central to the policy construction (e.g., purely descriptive regression analysis followed by a simulation-based analysis).

\begin{table}
\centering
\caption{Survey of Estimate-Then-Optimize Papers in \textit{Management Science}}
\label{tab:survey_results}
\begin{tabular}{lc}
\toprule
\textbf{Search Criteria} & \textbf{Count} \\
\midrule
(1) Initial Search Results (Keyword match, last 10 years) & 875 \\
(2) Potentially Relevant (AI Screened via GPT-5.2) & 111 \\
(3) Confirmed Estimate-Then-Optimize (Manual Review) & 55 \\
(4) Susceptible to Winner's Curse (Model-Based Evaluation) & 53 \\
\bottomrule
\end{tabular}
\medskip
\parbox{0.9\textwidth}{\small \textit{Notes:} Rows 1 and 2 represent a sample of papers identified only via their titles and abstracts. Row 3 represents papers confirmed via manual review of the full text to follow an estimate-then-optimize pipeline. Row 4 indicates the subset of Row 3 that rely on model-based evaluation.}
\end{table}

As summarized in Table~\ref{tab:survey_results}, this process yielded 55 papers. Of these, 53 (96\%) relied on model-based evaluation---specifically, they evaluated the performance of their proposed policy using the same data generating process or specific dataset used for estimation. As demonstrated in the rest of this paper, this methodology creates a direct channel for the winner's curse to inflate performance estimates. This survey suggests that the issue is not merely a theoretical curiosity but a pervasive feature of the current state of data-driven decision-making in the field.

\begin{table}[t]
\centering
\caption{Survey of model-based counterfactual practices in estimate-then-optimize papers}
\label{tab:survey_counterfactual_practices}
\begin{tabular}{l r}
\toprule
 & Count (\%) \\
\midrule
\textbf{Model-based evaluation} \\
Optim \& eval use same model family & 53 (100\%) \\
Optim \& eval use exact same model & 35 (66\%) \\
\addlinespace
\textbf{Corrections made for winner's curse} \\
Robust optimization & 3 (6\%)\\
De-biasing & 2 (4\%) \\
\addlinespace
\textbf{Justifications for model-based counterfactuals} \\
Accuracy & 25 (47\%) \\
Stability & 19 (36\%) \\
Calibration & 7 (13\%) \\
\addlinespace
\textbf{Applications of sample splitting}\\
Training subset evaluation (online) & 8 (15\%) \\
Prediction model testing & 11 (21\%) \\
Separate data for evaluation model & 4 (8\%)\\
\addlinespace
\textbf{Statistical Claims} \\
Instrumental variables used to estimate model & 5 (9\%)\\
Confidence interval reported & 9 (17\%) \\
\addlinespace
\textbf{Model class used for estimation} \\
Simple parametric statistical models & 35 (66\%) \\
Structural models & 13 (25\%) \\
Machine learning models & 5 (9\%) \\
\addlinespace
\textbf{Problem domain} \\
Revenue management & 24 (45\%) \\
Healthcare & 16 (30\%)\\
Service & 11 (21\%)\\
Finance/marketing & 2 (4\%) \\
\bottomrule
\end{tabular}
\vspace{0.5ex}
\begin{flushleft}
\footnotesize
Notes: “Simple parametric” includes linear/logit/MNL (both offline and online) as well as related classical econometric specifications; “structural” includes MDP/queueing/SIR; ``machine learning'' includes KNN/KDE/DQN.
\end{flushleft}
\end{table}

Table~\ref{tab:survey_counterfactual_practices} summarizes the specific characteristics of these 53 papers. We start with the mechanics of the evaluation. Thirty-five of the 53 papers (66\%) evaluate the performance of their optimized policy using the \textit{exact same} estimated model that was used to generate the policy. This is the most severe form of the winner's curse, as the optimizer is free to exploit every residual error in the model. The remaining papers generally use the same model family but re-estimate parameters, which, as we show in Section~\ref{sec:regularized}, does not eliminate the bias.

How do these papers justify this approach? Most do not. Among those that do, a tiny fraction apply the techniques discussed in Section~\ref{sec:rel_work} to mitigate the impact of the winner's curse on their model-based estimates. Three papers (6\%) acknowledged that their model parameters were uncertain and adopted a robust optimization approach. During policy learning, model parameters were adversarially selected from a ball around the initial parameter estimates, which were subsequently used in policy evaluation. However, these approaches rely on the strong assumption that the model is well-specified. Two other papers (4\%) acknowledged that their policy learning process led to an ``in-sample'' policy evaluation, and used held-out datasets to adjust counterfactual outcomes. These corrections are similar to training policy learning and evaluation models from two split samples, which we demonstrate in Sections~\ref{sec:simple_illustration} \&~\ref{sec:regularized} is insufficient to rule out spurious gains from the winner's curse unless the counterfactual models are well-specified and completely unregularized.

A larger fraction do not attempt to correct for the winner's curse, but provide some kind of justification. The most commonly cited defense is predictive accuracy---twenty-five papers (47\%) argue that because the underlying estimated model achieves high accuracy (e.g., low MSE on some, not always held-out, test set), the resulting policy evaluation is valid. Nineteen papers (36\%) appeal to stability (e.g., the estimated gains are stable across bootstrapped subsamples). Calibration of predictions is cited less frequently, appearing in only seven papers (13\%). A number of papers utilize some form of sample splitting to support these justifications. Eight papers (15\%) describe online policy learning algorithms, whose training set at any point is inherently a subset of the data used to train the evaluation model. Eleven papers (21\%) specifically hold out a test set to evaluate their prediction model, but do not utilize this held out data at any point of the policy evaluation process, instead directly reporting the training model's estimates of the optimized policy's performance. Four more papers (8\%) hold out a test set which they then separately use to train an evaluation model to evaluate the optimized policy's performance. Almost all of the works used observational data to train their models, with only a few exploiting quasi-experimental variation to avoid correlation between training assignments and observed outcomes; none used data from an RCT. As we demonstrate in Sections~\ref{sec:simple_illustration} \&~\ref{sec:regularized}, none of these justifications---accuracy, stability, calibration, or training policy evaluation models on held out data---are sufficient to rule out spurious gains from the winner's curse, even when the data comes from an RCT.

We also consider statistical claims such as causality or uncertainty quantification. Only five papers (9\%) utilized instrumental variables to deliver causal interpretations of their model's parameters from their historical data, and only nine papers (17\%) report confidence intervals for their policy improvements. None attempt to construct these intervals in a way that accounts for the optimization process itself. This general lack of any statistical inference is particularly troubling given the magnitude of the claimed improvements, which may rely on exploiting noisy predictions from tail events or rare subpopulations.

Two additional points are worth noting regarding the model classes and problem domains. First, the winner's curse is not limited to ``black-box'' machine learning methods, and in fact may be worse for more restricted model families such as linear or logit regression. Thirty-five of the papers (66\%) use simple parametric models (such as linear or logit regressions, either via offline or online learning), 13 (25\%) use structural models (such as MDPs or queuing), and only five (9\%) use black-box machine learning models (such as KNNs or DQNs).
Second, the practice is widespread across domains, with revenue management (45\%) and healthcare (30\%) being the most represented fields, likely due to the natural fit of optimization techniques in these areas. In short, our survey suggests that model-based evaluation is the default standard in the literature, despite its inherent susceptibility to the winner's curse.

\section{Mathematical Formulation of Winner's Curse}

While the winner's curse is well-known, it is typically formulated in the setting of selecting among a finite set of candidate policies. In these settings, policy evaluation is typically unbiased, so existing methods for correcting for the winner's curse work well. However, recent work considers much more complex policy classes---e.g., with continuous treatments and/or targeting treatments based on individual covariates.  In general, we consider a candidate policy $\hat\pi:\mathcal{X}\to\mathcal{T}$ mapping individual covariates $x\in\mathcal{X}$ to treatments $t\in\mathcal{T}$, with the goal of maximizing outcomes $y\in\mathcal{Y}\subseteq\mathbb{R}$. We assume we are given a historical dataset $\{(X_n,T_n^o,Y_n^o)\}_{n=1}^N$, where individual covariates $X_n\in\mathcal{X}$ are sampled from a distribution $p(x)$, treatments $T_n^o\in\mathcal{T}$ are sampled from a known historical policy $\pi^o(t\mid x)$, and outcomes $Y_n^o\in\mathcal{Y}$ are sampled from a distribution $p(y\mid x,t)$; we assume $(X_n,T_n^o,Y_n^o)$ are i.i.d. samples. Then, the goal of policy evaluation is to estimate the performance of $\hat\pi$:
\begin{align*}
J^*(\hat\pi)=\E_{X\sim p(\cdot),T\sim\hat\pi(\cdot\mid X),Y\sim p(\cdot\mid X,T)}[Y].
\end{align*}

The challenge is that the historical data consists of samples where the treatment is distributed according to $T^o\sim\pi^o(\cdot\mid X)$. A natural strategy for correcting for this distribution shift is to use importance weighting, which results in the generalized Horovitz-Thompson estimator \citep{horvitz1952generalization}. In particular, given any \emph{baseline function} $\hat{Q}:\mathcal{X}\times\mathcal{T}\to\mathbb{R}$, we have
\begin{align*}
\hat{J}_{\text{HT}}(\hat{\pi}) &= \frac{1}{N}\sum_{n=1}^N\left\{\frac{\hat{\pi}(T_n^o\mid X_n)}{\pi^o(T_n^o\mid X_n)}(Y_n^o-\hat{Q}(X_n,T_n^o))+\sum_{t\in\mathcal{T}}\hat\pi(t\mid X_n)\hat{Q}(X_n,t)\right\}.
\end{align*}
The most common choice for $\hat{Q}$ is an estimate of the conditional expected outcome:
\begin{align*}
Q^*(x,t) = \E_{Y\sim p(\cdot\mid X=x,T=t)}[Y].
\end{align*} 
We can estimate $Q^*$ on a second, held-out historical dataset using any machine learning algorithm:
\begin{align*}
\hat{Q}=\operatorname*{\arg\min}_Q\frac{1}{N'}\sum_{n=1}^{N'}(Q(X_n',T_n^{o\prime})-Y_n^{o\prime})^2.
\end{align*}
Taking $\hat{Q}\approx Q^*$ leads to the AIPW estimator~\citep{robins1994estimation,robins1995semiparametric}. AIPW provides an unbiased estimate of the true policy performance, enabling the application of standard asymptotic techniques to derive valid confidence intervals. Importantly, this guarantee holds for \emph{any} choice of baseline $\hat{Q}$, so the estimate $\hat{Q}$ can be arbitrarily inaccurate or biased without impacting validity. Instead, poor choices of $\hat{Q}$ only increase the variance of the AIPW estimator, which increases the size of the resulting confidence intervals.

Next, model-based evaluation simply uses $\hat{Q}$ to estimate counterfactuals:
\begin{align*}
\hat{J}_{\text{MB}}(\hat{\pi}) = \frac{1}{N} \sum_{n=1}^N \sum_{t\in\mathcal{T}} \hat{\pi}(t\mid X_n)\hat{Q}(X_n,t).
\end{align*}

Now, we discuss how the winner's curse arises in both of these approaches. The traditional winner's curse can arise in either approach if the policy is optimized on the same historical dataset as the one used for policy evaluation. We focus on the case where sample splitting is used, in which case AIPW is guaranteed to be unbiased and its standard confidence intervals are valid. For model-based evaluation, we further assume that separate predictive models are trained for policy optimization and evaluation, and that they are both trained on separate held-out datasets, thus requiring four separate samples. Even under these conditions, model-based evaluation remains biased. This bias can be characterized in terms of the residuals $R(x,t) = \hat{Q}(x,t)-Q^*(x,t)$ in the predictive model $\hat{Q}(x,t)$ used for policy evaluation:
\begin{align*}
\E[\hat{J}_{\text{MB}}(\hat\pi)]
&= \sum_{t\in\mathcal{T}} \E_{X\sim p(\cdot),Y\sim p(\cdot\mid X,t)}[\hat{\pi}(t\mid X)(Q^*(X,t)+R(X,t))] \\
&= J^*(\hat\pi) + \underbrace{\sum_{t\in\mathcal{T}} \E_{X\sim p(\cdot),Y\sim p(\cdot\mid X,t)}[\hat{\pi}(t\mid X)R(X,t)]}_{\text{bias}}
\end{align*}
In the traditional winner's curse~\citep{harrison1984decision,smith2006optimizer}, there is no sample splitting, so policy optimization tends to select for positive errors that arise from random noise---i.e., $\hat{J}_{\text{MB}}(\hat\pi)$ is biased even if $R(X,t)$ has mean zero. However, when the predictive model is biased, then this bias persists under sample splitting.

\section{Misspecification Causes the Winner's Curse} \label{sec:simple_illustration}

In this section, we provide a theoretical construction where misspecification drives the winner's curse even when justifications J1 (accurate, stable, and calibrated model), J2 (historical data uses random treatment assignment), and J4 (sample splitting) all hold. While our example is stylized, as we discuss in Section~\ref{subsec:connection_practice}, the key aspects of our example that drive the winner's curse---namely, the exploitation of estimation errors in regions of sparse data---correspond to natural behaviors of simple estimators (e.g., ordinary linear regression) in high-dimensional settings.

\subsection{Theoretical Analysis}

We consider a simple setting with a single real-valued treatment $t\in[0,1]$ and an outcome $y\in\mathbb{R}$, without covariates. We consider a noiseless true model $y=f^*(t)$, where $f^*$ is a piecewise linear function with two hyperparameters $t_0\in(0,1)$ and $y_{\text{max}}\in\mathbb{R}_{>0}$:
\begin{align*}
f^*(t)=\begin{cases}
\frac{ty_{\text{max}}}{t_0}&\text{if }t\le t_0 \\
\frac{(1-t)y_{\text{max}}}{1-t_0}&\text{if }t>t_0.
\end{cases}
\end{align*}
This function is illustrated as the solid black line in Figure~\ref{fig:example}.

\begin{figure}
\centering
\includegraphics[width=0.7\textwidth]{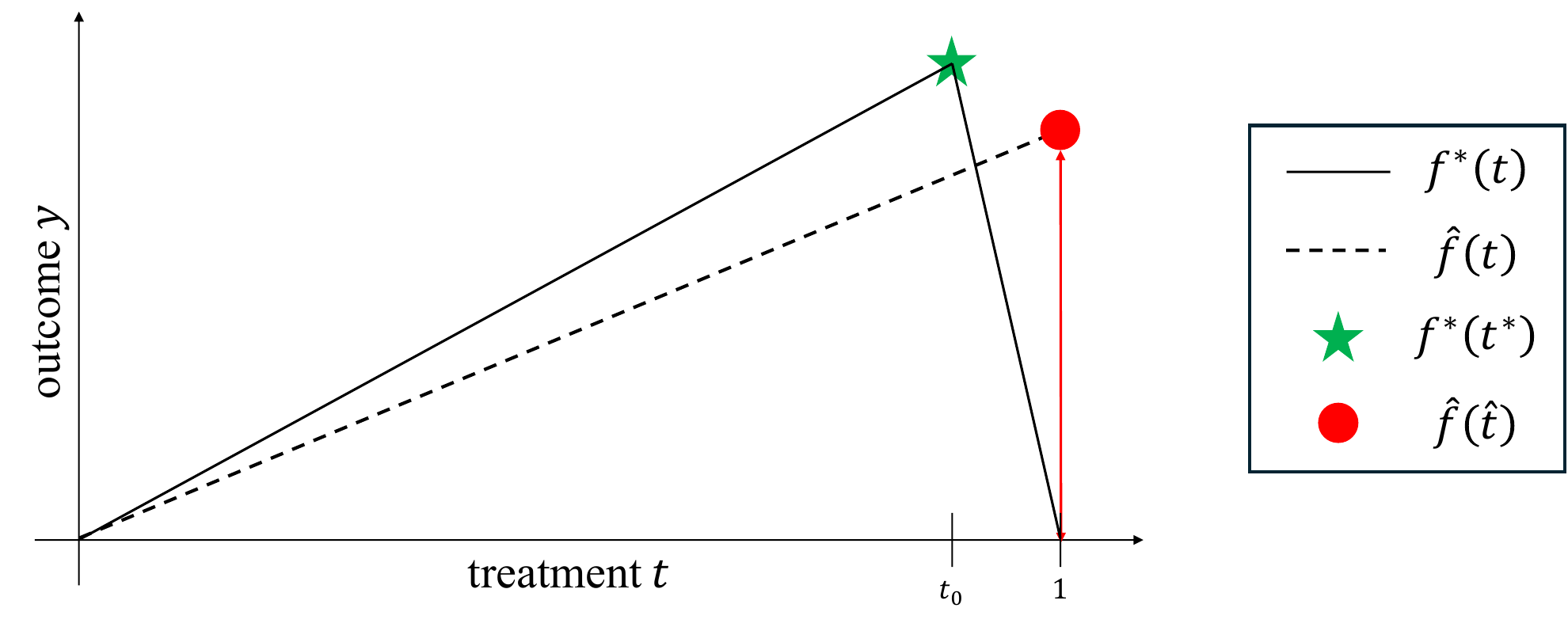}
\caption{Illustration of stylized example. The $x$-axis is the treatment $t$, the $y$-axis is the performance outcome $y$, the solid black line is $f^*(t)$, the dashed black line is $\hat{f}(t)$, the green star denotes $f^*(t^*)$ at $t^*=\operatorname*{\arg\max}_{t\in[0,1]}f^*(t)$, the red circle denotes $\hat{f}(\hat{t})$ at $\hat{t}=\operatorname*{\arg\max}_{t\in[0,1]}\hat{f}(t)$, and the red line denotes the optimistic bias $\hat{f}(\hat{t})-f^*(\hat{t})$ due to the winner's curse.}
\label{fig:example}
\end{figure}

To estimate this model, we consider historical data collected using a uniform distribution over treatments $t\sim\text{Uniform}([0,1])$. Specifically, a single example in our historical dataset is a pair $(t,y)$, where $y=f^*(t)$. For simplicity, we consider the infinite data limit, where we observe the exact distribution over $t$. We assume we are fitting a linear function $\beta t$, a standard model family that achieves high out-of-sample accuracy in this setting.
\begin{lemma} \label{lem:simple-1}
The ordinary least squares (OLS) estimate of $\beta\in\mathbb{R}$ is
\begin{align*}
\hat{\beta}=\frac{(1+t_0)y_{\text{max}}}{2}.
\end{align*}
\end{lemma}
\proof{Proof of Lemma~\ref{lem:simple-1}}
Note that the OLS estimator is $\hat\beta=\mathbb{E}[t^2]^{-1}\mathbb{E}[tf^*(t)]$. We have
\begin{align*}
\mathbb{E}[t^2]=\int_0^1t^2dt=\frac{1}{3},
\end{align*}
and
\begin{align*}
\mathbb{E}[tf^*(t)]
&=\int_0^{t_0}\frac{t^2y_{\text{max}}}{t_0}dt+\int_{t_0}^1\frac{t(1-t)y_{\text{max}}}{1-t_0}dt \\
&=\frac{t_0^2y_{\text{max}}}{3}+\int_0^{1-t_0}\frac{s(1-s)y_{\text{max}}}{1-t_0}ds \\
&=\frac{t_0^2y_{\text{max}}}{3}+\frac{(1-t_0)y_{\text{max}}}{2}-\frac{(1-t_0)^2y_{\text{max}}}{3} \\
&=\frac{(1+t_0)y_{\text{max}}}{6}.
\end{align*}
The claim follows. \Halmos

Thus, our model is $\hat{f}(t)=\hat\beta t$; this function is visualized as the dashed line in Figure~\ref{fig:example}. Our first result characterizes the winner's curse for this model.
\begin{proposition} \label{prop:simple-1}
The optimistic bias from the winner's curse is $\frac{(1+t_0)y_{\text{max}}}{2}$.
\end{proposition}
\proof{Proof of Proposition~\ref{prop:simple-1}}
The best treatment according to our model is $\hat{t}=\operatorname*{\arg\max}\hat{f}(t)=1$. It has estimated outcome $\hat{y}=\hat{f}(\hat{t})=\frac{(1+t_0)y_{\text{max}}}{2}$ and true outcome $\tilde{y}=f^*(\hat{t})=0$, so the bias is $\hat{y}-\tilde{y}=\frac{(1+t_0)y_{\text{max}}}{2}$. \Halmos

For $t_0\ge\frac{1}{2}$, we have $\hat{y}-\tilde{y}\ge\frac{y_{\text{max}}}{2}$; thus, the winner's curse can be arbitrarily large. Note that the estimated performance is comparable to the true optimal performance (the optimal treatment is $t^*=\operatorname*{\arg\max}_{t\in[0,1]}f^*(t)=t_0$ and its true outcome is $y^*=f^*(t^*)=y_{\text{max}}$). Yet, the realized performance is zero—i.e., the estimated outcome is entirely illusory due to the winner's curse.

Crucially, this happens even though the model is highly accurate on the historical data
distribution (i.e., achieving arbitrarily low MSE).
\begin{proposition} \label{prop:simple-2}
Letting $t_0=1-\epsilon$, if $\epsilon\le\frac{1}{2}$, then the mean-squared error satisfies
\begin{align*}
\text{MSE}(\hat{f})\le\frac{\epsilon y_{\text{max}}^2}{3}.
\end{align*}
\end{proposition}
\proof{Proof of Proposition~\ref{prop:simple-2}}
Note that
\begin{align*}
\text{MSE}(\hat{f})
&=\int_0^{1-\epsilon}\left(\frac{(2-\epsilon)ty_{\text{max}}}{2}-\frac{ty_{\text{max}}}{1-\epsilon}\right)^2dt+\int_{1-\epsilon}^1\left(\frac{(2-\epsilon)ty_{\text{max}}}{2}-\frac{(1-t)y_{\text{max}}}{\epsilon}\right)^2dt \\
&=\frac{\epsilon(4-\epsilon)y_{\text{max}}^2}{12} \\
&\le\frac{\epsilon y_{\text{max}}^2}{3},
\end{align*}
as claimed. \Halmos

Intuitively, the discrepancy between high accuracy and large winner's curse occurs because the optimized treatment induces a shift in the data distribution---the estimated model performs significantly worse on this shifted distribution.

\subsection{Contextualizing Our Result}

Our model illustrates that even when justifications J1, J2, and J4 in Section~\ref{sec:intro} hold, the winner's curse can result in arbitrarily large, spurious policy improvement estimates. Here, we connect our example to these justifications:
\begin{itemize}
\item \textbf{J1: Accurate, stable, and calibrated:} Our model $\hat{f}$ is perfectly stable since we are in the limit of infinite data. Furthermore, as $t_0\to1$ (equivalently, $\epsilon\to 0$), it is highly accurate. A highly accurate model is by construction well-calibrated, since it can simply always express confidence in its prediction.
\item \textbf{J2: Random treatment assignment:} Our historical data is generated using random treatment assignment (i.e., uniformly random).
\item \textbf{J4: Sample splitting:} Because we take the limit of infinite data, we can retrain the model under independent data splits and the estimated models on each split will be identical. Thus, sample splitting does not affect the policy improvement estimate.
\end{itemize}

\subsection{Connection to Practice}
\label{subsec:connection_practice}

While our example is highly stylized, it proves that the justifications described in Section~\ref{sec:intro} (except the well-specified justification, J3) are insufficient to rule out the winner's curse from a theoretical perspective. To connect this stylized example to practical failure modes, consider the two key factors that drive the winner's curse:
\begin{enumerate}
\item Because the model family is misspecified, the estimated model produces biased counterfactual estimates even in the limit of infinite data.
\item The bias occurs in a small region of the input space (i.e., $t\in(t_0,1]$), allowing the model to achieve very low MSE on the observed data distribution.
\end{enumerate}
Then, the optimizer can exploit the high-error region of the input space to maximize the predicted outcome, leading to the winner's curse.

One notable feature of our example is that the outcome goes from $y_{\text{max}}$ to zero with a tiny change in treatment, which might be unlikely to happen in practice. However, this construction was only necessary since we were using a simple, one-dimensional treatment and no covariates. In realistic examples, the input space is typically high-dimensional, providing significantly more room for this kind of bias to occur. For instance, in the refugee setting, the goal is to assign refugees to cities to maximize employment rate (e.g., after 90 days). In this problem, each city is a discrete treatment $t\in\mathcal{T}=\{1,...,k\}$. This problem additionally involves targeting the treatment based on individual covariates $x\in\mathcal{X}\subseteq\mathbb{R}^d$, such as country of origin, languages spoken, etc. The outcome for a single individual is an indicator for employment $y\in\{0,1\}$. In this case, a policy has form $\pi:\mathcal{X}\to\mathcal{T}$. A standard way to learn such a policy called the \emph{plug-in approach} is to first estimate a model $f^{\beta}:\mathcal{X}\times\mathcal{T}\to[0,1]$ mapping covariate-treatment pairs to outcome probabilities. This function induces a policy
\begin{align*}
\pi^{\beta}(x)=\operatorname*{\arg\max}_{t\in\mathcal{T}}f^{\beta}(x,t),
\end{align*}
i.e., choose the best treatment according to our estimated model.

Consider using a low-capacity model family---e.g., linear models $f^{\beta}(x,t)=\beta^\top\phi(x,t)$ over a feature map $\phi(x,t)\in\mathbb{R}^m$. If the feature map is too simple, then even for the optimal linear model $\beta^*$, some covariate-treatment pairs will naturally be upwards biased---i.e., $f^{\beta^*}(x,t)>\mathbb{E}[y^*\mid x,t]$, where $y^*$ is the true employment outcome---and some will naturally be downwards biased. Even if this average bias is small (or zero), the optimizer can systematically exploit the upwards-biased covariate-treatment pairs to obtain spurious improvements. Furthermore, as the feature dimension $m$ becomes larger, we might expect the winner's curse to become larger since there are more opportunities for bias to arise.

\section{Regularization Causes the Winner's Curse}
\label{sec:regularized}

In Section~\ref{sec:simple_illustration}, we showed that misspecification can result in the winner's curse. Thus, one may hope that the winner's curse is mitigated when using a high-dimensional (e.g., LASSO) or non-parametric model (e.g., random forest); then, we may argue that it is likely that the true data generating process lies in the model family, satisfying J3.

Unfortunately, we show that the winner's curse can still arise when the estimation problem is well-specified as long as the model is \textit{regularized}. This case is substantially more delicate---indeed, it is true that in the well-specified setting (along with the other justifications), the winner's curse disappears asymptotically. However, in finite sample, these models must be regularized to avoid arbitrary overfitting, which creates bias in finite sample. Surprisingly, this bias can actually be \emph{stable}---i.e., the estimated model is biased in the same way with high probability over data sub-samples. Intuitively, regularization is typically applied in a systematic way, forcing ``high variability'' components of the model family to zero until a sufficient amount of data is available to estimate them accurately.

\subsection{Theoretical Analysis}

We show via a simple well-specified linear model that in fact, regularization bias is sufficient to result in the winner's curse. Specifically, consider a linear model with a binary treatment $t\in\{0,1\}$, a single continuous covariate $x\in\mathbb{R}$, and an interaction term $tx$. We assume that $t\sim\text{Bernoulli}(p)$ (for some $p\in[0,1]$) and $x\sim\mathcal{N}(0,1)$ are independent random variables. Also, assume that the noise distribution is $\epsilon\sim\mathcal{N}(0,1)$. We can interpret this setting in the context of our refugee matching case study by considering two locations $A,B$, with $t$ indicating whether to assign a refugee to location $A$ (if $t=0$) or $B$ (if $t=1$). Each individual is represented by a single covariate $x$.

To estimate a useful targeting model, we need to have interaction terms between $t$ and $x$; otherwise, a linear model will conclude that a single treatment is best for all individuals. Thus, we consider the following feature map (which we assume to be well-specified):
\begin{align*}
\phi(x,t)=\begin{bmatrix}t&x&tx\end{bmatrix}^\top,
\end{align*}
which includes a single interaction term $tx$. In this section, we distinguish between \emph{covariates} (i.e., the individual attribute $x$) and \emph{features} (i.e., the components of the feature map $\phi(x,t)$ encoding covariate-treatment pairs into a real-valued vector). Next, for $b>0$, we assume the true parameters are
\begin{align*}
\beta^*=\begin{bmatrix}0&b&0\end{bmatrix}^\top.
\end{align*}
Note that under these parameters, there is no value to targeting---the outcome only depends on $x$, and not on the interaction $tx$.

Now, suppose that $p\approx1$; that is, in our historical data, refugees are overwhelmingly assigned to location $B$ (e.g., due to stringent capacity constraints in $A$). In this case, there are very few refugees assigned to location $A$. In the data, we observe that the outcome $y$ is strongly correlated with both $x$ and $tx$ (since $x$ and $tx$ are also strongly correlated), but we cannot distinguish which of these two features is causal; thus, we cannot identify their coefficients. Mathematically, this is because the true covariance matrix
\begin{align*}
\Sigma
=\mathbb{E}_{x,t}[\phi(x,t)\phi(x,t)^\top]
=\begin{bmatrix}
\mathbb{E}_{x,t}[t^2] & \mathbb{E}_{x,t}[tx] & \mathbb{E}_{x,t}[t^2x] \\
\mathbb{E}_{x,t}[tx] & \mathbb{E}_{x,t}[x^2] & \mathbb{E}_{x,t}[tx^2] \\
\mathbb{E}_{x,t}[t^2x] & \mathbb{E}_{x,t}[tx^2] & \mathbb{E}_{x,t}[t^2x^2]
\end{bmatrix}
=\begin{bmatrix}
p & 0 & 0 \\
0 & 1 & p \\
0 & p & p
\end{bmatrix},
\end{align*}
has $\lambda_{\text{min}}(\Sigma)\to0$ as $p\to1$, highlighting the strong correlation between $x$ and $tx$. As a consequence, any \textit{unbiased} estimate of $\beta^*$ (i.e., via OLS) will be highly unstable. Specifically, while we can accurately estimate $\beta^*_1$ from just a few samples, estimating $\beta_2^*,\beta_3^*$ requires many more samples due to insufficient variation in the training data. This is not a problem for prediction---we can still identify the sum $\beta_2^*+\beta_3^*=b$. If our goal is only to achieve good out-of-sample accuracy on the distribution of the historical data, then identifying $\beta_2^*+\beta_3^*$ is sufficient, since $t=1$ with high probability, so
\begin{align*}
\beta_2^*x+\beta_3^*xt=(\beta_2^*+\beta_3^*)x=(\beta_2^*+\beta_3^*)tx.
\end{align*}
We can obtain a stable estimate of $\beta_2^*+\beta_3^*$ by using regularization, specifically, ridge regression. Intuitively, regularization forces the ridge regression estimator to evenly distribute $b$ between $\beta_2^*$ and $\beta_3^*$; indeed, consider the parameter vector
\begin{align*}
\bar\beta^{(\lambda)}=\begin{bmatrix}0&b/2&b/2\end{bmatrix}^\top.
\end{align*}
We can show that the ridge regression estimator $\hat\beta^{(\lambda)}$ concentrates tightly to $\bar\beta^{(\lambda)}$, implying that it is stable.
Furthermore, $\bar\beta^{(\lambda)}$ satisfies the desired relation $\bar\beta^{(\lambda)}_2+\bar\beta^{(\lambda)}_3=b$; thus, this estimator has high accuracy. Specifically, define the \emph{excess mean-squared error (MSE)} to be the generalization error relative to the true model (typically estimated on a held-out test set):
\begin{align*}
\mathcal{E}(\beta)=\mathbb{E}_{x,t}[(f^{\beta}(x,t)-f^{\beta^*}(x,t))^2],
\end{align*}
where $\mathbb{E}_{x,t}$ is the expectation with respect to the training covariates $X\sim\mathcal{N}(0,1)^n$ and training treatments $T\sim\text{Bernoulli}(p)^n$. Then, $\mathcal{E}(\bar\beta^{(\lambda)})$ can be small (e.g., on the order of $1/\sqrt{n}$) even when $\|\bar\beta^{(\lambda)}-\beta^*\|_2$ is large (e.g., on the order of $1$).

However, our goal is policy learning---thus, it is not enough to achieve good accuracy on the observed data distribution. Instead, we need to obtain good accuracy on the eventual data distribution induced by using an alternative policy $\pi:\mathcal{X}\to\mathcal{T}$ to assign treatments. Unfortunately, this requires that we identify $\beta_2^*$ and $\beta_3^*$ separately; the winner's curse arises from this disconnect. In particular, note that $\pi^{\bar\beta^{(\lambda)}}(x)=\mathbbm{1}(x\ge0)$. Because the interaction term $tx$ has estimated coefficient $\bar\beta^{(\lambda)}_3=b/2$, the policy believes it is beneficial for units with $x\ge0$ to receive treatment $t=1$; conversely, for $x\le0$, it believes it is beneficial for units to receive treatment $t=0$.
\begin{definition}
\rm
Given $\beta,\beta'\in\mathcal{B}$, the \emph{policy improvement} of $\beta$ when evaluated using $\beta'$ is
\begin{align*}
\mathcal{P}(\beta;\beta')=\mathbb{E}_x[f^{\beta'}(x,\pi^{\beta}(x))],
\end{align*}
and the \emph{optimism bias} is 
\begin{align*}
\mathcal{R}(\beta;\beta')=\mathcal{P}(\beta;\beta')-\mathcal{P}(\beta;\beta^*).
\end{align*}
\end{definition}
The optimistic bias of $\pi^{\bar\beta^{(\lambda)}}$ is zero for $x\ge0$ and $-bx/2>0$ for $x<0$, so its overall optimistic bias is
\begin{align*}
\mathcal{R}(\bar\beta^{(\lambda)};\bar\beta^{(\lambda)})
=\mathbb{E}_x\left[\frac{bx}{2}+\frac{bx\mathbbm{1}(x\ge0)}{2}-bx\right]
=-\frac{b}{2}\cdot\mathbb{E}_x[x\mathbbm{1}(x<0)]
=\frac{b}{2\sqrt{2\pi}}.
\end{align*}
Furthermore, note that stability of $\hat\beta^{(\lambda)}$ implies stability of the optimism bias (formalized in Proposition~\ref{prop:biasstability}); thus, using sample splitting or a bootstrapped estimate would reliably produce the same optimism bias. We formalize these results in the following proposition.
\begin{proposition}
\label{prop:regularized}
Given $b>0$, $\alpha\ge32$, and $\delta\in\mathbb{R}_{>0}$, let $\Delta=\sqrt{120\log(144/\delta)/n}$, $\lambda=\alpha\Delta$, and $p=1-\Delta/2$. Assume that $\lambda\le1$ and
$n^{1/4}\ge4\sqrt{\log(144/\delta)}/b$. Then, with probability at least $1-\delta$, all of the following hold:
\begin{itemize}
\item \textbf{Accuracy:} The MSE of our estimator satisfies
\begin{align*}
\mathcal{E}(\hat\beta^{(\lambda)})
\le2b^2\alpha\sqrt{\frac{120\log(144/\delta)}{n}}+\frac{30\log(144/\delta)}{n}.
\end{align*}
\item \textbf{Stability:} Our estimated model satisfies
\begin{align*}
\|\hat\beta^{(\lambda)}-\hat\beta^{(\lambda)\prime}\|_2&\le\frac{8b}{\alpha}+4\sqrt{\frac{\log(144/\delta)}{\alpha\sqrt{n}}},
\end{align*}
where $\hat\beta^{(\lambda)}$ and $\hat\beta^{(\lambda)\prime}$ are trained on i.i.d. datasets. Our estimated policy improvement satisfies
\begin{align*}
|\mathcal{P}(\hat\beta^{(\lambda)};\hat\beta^{(\lambda)\prime})-\mathcal{P}(\hat\beta^{(\lambda)\prime\prime};\hat\beta^{(\lambda)\prime\prime\prime})|
\le2R,
\end{align*}
where
\begin{align*}
R=\frac{12b}{\alpha}+6\sqrt{\frac{\log(144/\delta)}{\alpha\sqrt{n}}}+\frac{480b}{\alpha^2}+\frac{120\log(144/\delta)}{\alpha b\sqrt{n}},
\end{align*}
where $\hat\beta^{(\lambda)\prime\prime}$ and $\hat\beta^{(\lambda)\prime\prime\prime}$ are similarly trained on i.i.d. datasets.
\item \textbf{Winner's curse:} The optimism bias satisfies
\begin{align*}
\mathcal{R}(\hat\beta^{(\lambda)};\hat\beta^{(\lambda)\prime})
\ge\frac{b}{3\sqrt{2\pi}}-R.
\end{align*}
\end{itemize}
\end{proposition}
We give a proof in Appendix~\ref{sec:prop:regularized:proof}. If we take $\alpha=n^{1/4}$, then for any choice of $b$, as $n$ becomes sufficiently large,  the MSE goes to zero and the estimator is perfectly stable, whereas the optimism bias remains bounded from below by $b/(3\sqrt{2\pi})$. Thus, accuracy and stability cannot definitively rule out the winner's curse.

\subsection{Contextualizing Our Result}

Our model illustrates that even when all of the justifications in Section~\ref{sec:intro} hold, the winner's curse can result in arbitrarily large, spurious policy improvement estimates. Here, we connect our example to these justifications:
\begin{itemize}
\item \textbf{J1: Accurate, stable, and calibrated:} Our model $\hat\beta^{(\lambda)}$ is accurate (since $\mathcal{E}(\hat\beta^{(\lambda)})$ is small for sufficiently large $n$), stable (since it concentrates for sufficiently large $n$ and taking $\alpha=n^{1/4}$), and well-calibrated (since it can always express valid confidence in its prediction).
\item \textbf{J2: Random treatment assignment:} Our historical data is generated using random treatment assignment (i.e., randomly assign to location $A$ with probability $1-p$ and to location $B$ with probability $p$, where $p$ is known).
\item \textbf{J3: Well-specified model family:} The model family contains the true model $\beta^*$.
\item \textbf{J4: Sample splitting:} Our result uses sample splitting for policy evaluation.
\end{itemize}
Yet, by taking $b$ large, we can obtain arbitrarily large unwarranted optimism in our policy improvement estimates.

\subsection{Connection to Practice}

While our example is stylized, it proves that all of the justifications described in Section~\ref{sec:intro} are insufficient to rule out the winner's curse from a theoretical perspective. In this case, the winner's curse is driven by a very similar mechanism as in Section~\ref{sec:simple_illustration}; the key difference is that the bias arises from regularization:
\begin{enumerate}
\item Because the model family is regularized, the estimated model produces stable but biased counterfactual estimates.
\item The bias occurs in a low-probability region of the input space (i.e., $t=0$), allowing the model to achieve very low MSE on the historical data distribution.
\end{enumerate}
As before, the optimizer can exploit the high-error region of the input space to maximize the predicted outcome, leading to the winner's curse.

In realistic scenarios, low-probability regions of the input space can include not just rare treatments (as in our stylized example), but also individual covariates that rarely occur (e.g., in our refugee example, a rare country of origin). Our setting is very low-dimensional and used a single, very rare treatment. In practice, the feature dimension $m$ is usually much larger, so it is actually much more likely that there are features $\phi(x,t)_i$ along which variation in the historical data is low, thereby requiring regularization to avoid instability and ensure accuracy. For these features, the optimizer can exploit biases that arise from regularization to achieve spurious improvements.

Beyond linear models, model families such as random forests and gradient boosted machines (GBMs) also perform implicit regularization to avoid overfitting the data. For instance, random forests average over a large number of decision trees, thereby regularizing estimates to zero when uncertain. Alternatively, GBMs rely on fitting low-capacity base models such as shallow decision trees, thereby inheriting the biases of the base model family. Thus, these kinds of models may also exhibit the winner's curse; our simulation study in Section~\ref{sec:refugee} illustrates this possibility.

\section{Case Study on Refugee Matching}
\label{sec:refugee}

We now provide empirical evidence that the justifications described in Section~\ref{sec:intro} do not guarantee accurate policy evaluation using a synthetic environment modeled after the refugee matching problem. Resettlement agencies attempt to settle incoming refugees in locations that maximize their welfare, often measured in the short-run by employment outcomes. To improve their matching of refugees to locations, \cite{bansak2018improving} and \cite{ahani2021placement} have proposed optimizing future location assignments using employment prediction models trained on historical outcomes. Both papers claim this procedure leads to large gains in employment, using model-based evaluation and justifications J1, J2, and J3. We design a simulation environment calibrated to 
the U.S.-based setting studied by~\cite{bansak2018improving}, but generate employment outcomes in such a way that no assignment policy can impact the expected employment rate. We train multiple prediction models, use them to estimate optimal policies, then evaluate these policies using both model-based methods (biased by the winner's curse) and a model-free method (IPW, which is unbiased, but high variance). Despite being built off accurate, calibrated, and stable prediction models (J1), being trained on random historical assignments (J2), and (for some models) being well-specified (J3), all model-based methods falsely report improved employment rates due to the winner's curse.

\subsection{Background}
\label{sec:sim_background}

The refugee matching problem provides a structure to analyze and improve assignments of incoming refugees to locations for resettlement. Resettlement agencies assign refugees to locations based on a small set of visible features and a variety of constraints. The authors argue that these constraints introduce random variation to the assignment process that enable cross-location comparisons of refugees with similar features. The goal of refugee matching is to identify feasible assignment policies that improve short-term refugee employment outcomes.

Upon entering the United States, incoming refugees are assigned to a resettlement agency that manages their arrival. Each resettlement agency sets up affiliate networks in various locations around the country to help integrate arriving refugees and provide them various services. Using a limited set of features for each refugee (such as age, nationality, and education), agencies must assign refugees to these locations.

Refugee assignments must satisfy various constraints. Some refugees have prior ties in the United States that dictate their placement. Other refugees require specific services, such as English language education, that only exist in some locations' affiliate networks. Moreover, each location has limited resources, creating capacity constraints which may vary over time. These constraints complicate the assignment process, but provide the variation required to estimate the heterogeneous impact of locations on refugees.

Resettlement agencies want to assign refugees in a welfare-maximizing way, subject to their various constraints. In the short-run, few welfare-relevant outcomes are easily trackable, so many agencies record employment outcomes after 90-days as a proxy for a successful refugee-location match. The goal is thus to use past observations of refugee-location matches and their employment outcomes to design an assignment policy that leads to a higher refugee employment rate.

\subsection{Optimizing Refugee Matches}
\label{sec:sim_optim}

Following \cite{bansak2018improving} and \cite{ahani2021placement}, we formulate a simple offline version of the refugee matching problem and describe an estimate-then-optimize method for proposing refugee assignments. Each incoming refugee consists of (known) covariates and (unknown) counterfactual employment outcomes for each location. We use the covariates to estimate employment outcomes for each location, then use an integer program to assign refugees to locations based on these estimates. Additional details are provided in Appendix~\ref{sec:add_deets}.

The assignment problem consists of matching $N$ unrestricted refugees to $L$ locations in order to maximize the number of employed refugees. Each refugee $i \in [N]$ is associated with a set of potential outcomes $Y_i(t) \in \{0,1\} \ \forall \ t \in [L]$ which describe whether the refugee would find employment at location $t$.\footnote{To match the assumption by the original authors, we also assume the standard Stable Unit Value Treatment Assumption which says one refugees potential outcomes are not impacted by other refugees' location assignments.} A matching $\pi$ assigns refugees to locations where $\pi_{it}\in\{0,1\}$ records whether refugee $i$ is assigned to location $t$. Each location $t$ can support a maximum of $c_t$ refugees. The optimal matching can then be described by the following integer program.
\begin{align*}
    \underset{\pi}{\text{maximize}} &\sum_{i = 1}^N \sum_{t=1}^L Y_{i}(t) \pi_{it}\\
    \text{s.t. }& \sum_{t = 1}^L \pi_{it} = 1 \ \forall \ i \in [N]\\
    & \sum_{i = 1}^N \pi_{it} \leq c_t \ \forall \ t \in [L]\\
    &\pi_{it} \in \{0,1\} \ \forall \ i \in [N], t \in [L].
\end{align*}

However, the counterfactual outcomes under alternative locations for each refugee are unknown. Thus, the authors use a predictive model to estimate these counterfactuals based on refugee covariates $X_i$ (such as age, nationality, and education). Specifically we have access to observations of prior refugee-location matches where for refugee $i$ we observe their covariates $X_i$, the location they were assigned to $T_i \in [L]$, and their observed employment outcome $Y_{i}(T_i)$.\footnote{We similarly adopt the assumption of unconfoundedness: conditional on the covariates $X_i$ the assignment $T_i$ is independent of the potential outcomes. Later we will assume the exact policy used to generate these matches is known and also assigns positive treatment propensity to each location.} These prior observations may have had assignment restrictions, which should be captured by the covariates $X_i$. We can use these observations to build a prediction model $\hat{\beta}$ where $\hat{\beta}_t(X_i)$ estimates $\E[Y_i(t)|X_i]$ given refugee covariates $X_i$ and a location $t$. We then propose using the assignment which maximizes the expected number of employments according to the model $\hat{\beta}$, which can be described by the following integer program.
\begin{align*}
    \underset{\pi}{\text{maximize}} &\sum_{i = 1}^N \sum_{t=1}^L \hat{\beta}_{t}(X_i) \pi_{it}\\
    \text{s.t. }& \sum_{t = 1}^L \pi_{it} = 1 \ \forall \ i \in [N]\\
    & \sum_{i = 1}^N \pi_{it} \leq c_t \ \forall \ t \in [L]\\
    &\pi_{it} \in \{0,1\} \ \forall \ i \in [N], t \in [L].
\end{align*}
We refer to the assignment given by this program as $\pi^{\hat{\beta}}$.

\subsection{Evaluating Refugee Matches}
\label{sec:sim_eval}

\cite{bansak2018improving} and \cite{ahani2021placement} both evaluate their proposed policies using model-based methods---specifically, they use the \textit{same} model for both policy estimation and policy optimization, which is prone to the winner's curse as discussed earlier. In this setting, under their assumptions of unconfoundedness and random assignment, a model-free approach (such as IPW), is guaranteed to be valid; however it is prone to high variance due to the size of the action space.

The model-based method of policy evaluation addresses the issue of missing counterfactual outcomes by imputing them. Specifically, it  uses the model $\hat{\beta}$ (which optimized the proposed assignment), effectively reporting $\sum_{i = 1}^N \sum_{t=1}^L \hat{\beta}_{t}(X_i) \pi^{\hat{\beta}}_{it}$. \cite{ahani2021placement} do the same, but also justify the evaluation using a set of $B$ alternative models $\{\hat{\beta}^{(b)}\} _{b=1}^B$ where model $b$ is built off an independently resampled bootstrap of the original training samples. While straightforward, these model-based approaches may produce biased evaluations if the prediction errors of the model used to propose the assignment $\pi^{\hat{\beta}}$ are correlated with the prediction errors of the model which evaluates the assignment. This is trivially true when $\hat{\beta}$ is reused for evaluation, but may occur even (i) if the evaluation models are built off bootstrapped samples of the training data or (ii) if the evaluation models are built off a completely independent sample. This can occur due to the specification or regularization the prediction models must apply to provide stable insights when training samples are limited. 

IPW is a model-free method of policy evaluation which addresses the issue of missing counterfactual outcomes by only evaluating the proposed policy $\pi^{\hat{\beta}}$ when the employment outcome is available. That is we assume each refugee $i \in [N]$ was actually assigned to location $T_i$ resulting in employment outcome $Y_i(T_i)$ where $T_i$ was generated independently from a known distribution where $\Pr(T_i = t) = p_{it}>0$ for all locations $t\in[L]$. IPW then rescales the outcome data based on how likely each assignment was to occur in the data according to $p$ and estimates the number of employed refugees using the cases the proposed assignment aligns with the observed assignment $\pi^{\hat{\beta}}_{iT_i} = 1$. Thus the expected number of employments under $\pi^{\hat{\beta}}$ according to IPW is
$$\sum_{i:\pi^{\hat{\beta}}_{iT_i}=1} \frac{Y_i(T_i)}{p_{iT_i}}.$$
This estimate of the proposed assignment's impact is unbiased, but can suffer from excessive variance if $p_{iT_i}$ is small for some refugees. In our environment there are 43 locations, leading to a volatile IPW estimator.

\subsection{Synthetic Environment}
\label{sec:sim_env}

We develop a synthetic environment to explore the potential impact of the winner's curse bias on model-based evaluation methods in the refugee matching problem. We tune our setting to mimic the US-based setting of \cite{bansak2018improving} and J2: we generate refugee covariates according to the marginal distributions reported by their supplemental material, then assign the same number of historical locations randomly in their observed proportions. We generate counterfactual employment outcomes in a way that depends on both refugee covariates and locations, but not their interaction, preventing any assignment policy from improving the expected employment rate.

The synthetic environment generates observations of prior refugee-location matches suitable for the policy optimization and evaluation methods described above. That is for each refugee $i$, our environment generates covariates $X_i$, a location $T_i$, and an employment outcome $Y_i(T_i)$. First, the environment assigns the covariates used by \cite{bansak2018improving} (age, gender, education, English-speaking, case restriction, country of origin, arrival year, and arrival month) independently at random according to the marginal distributions reported in their Supplemental Material. Next the location $T_i$ is drawn independently at random from a set of 43 possible locations according to the empirical distribution $p$ that \cite{bansak2018improving} reports. We report the covariate and location distributions in Appendix~\ref{apx:marginals}. Finally the employment outcome is drawn independently according to $Y_i(T_i)\sim \mathrm{Ber}(f(X_i,T_i))$ for our causal model $f$, which dictates the employment probability for a refugee with covariates $X_i$ placed in location $T_i$.

We generate employment outcomes using a causal model that induces the same expected employment rate for all feasible refugee-location matchings when all locations are at capacity. Specifically we determine the probability a refugee attains employment as
$$f(X_i,T_i) =\frac{1}{2}f_X(X_i) +\frac{1}{2}f_L(T_i),$$
where $f_X\in[0,1]$ represents the impact of the refugee's covariates on employment and $f_L \in[0,1]$ represents the impact of the location on employment. Under a causal model of this form no feasible matching can impact the expected employment rate if all locations are at capacity. We may rearrange the probability of employment across refugees through different assignments, but the lack of covariate-location interaction effects prevents us from improving the employment rate via better matchings. The refugee effect, $f_X$, is determined by first using a logit regression with random coefficients\footnote{We manually assign the covariate indicating case restriction=free to have a large positive impact. We do this as the test set of~\cite{bansak2018improving} only includes free case placements and achieves $34\%$ employment while their training set includes both types of cases and achieves $23\%$ employment. To replicate this difference we gave free case placement a positive impact on employment probability.} over refugee covariates and locations to assign an employment outcome to each refugee, then training a random forest regression using the refugee covariates to determine the underlying refugee effect. The location effect, $f_L$, for each of the 43 locations is drawn independently from a Beta distribution with $\alpha = 1$ and $\beta = 2$.\footnote{We chose these parameters to replicate the employment differences~\cite{bansak2018improving} observes across locations as faithfully as possible. Employment rates by location are reported in Appendix~\ref{apx:emp_loc}.}

\subsection{Simulation Details}
\label{sec:sim_details}

We test the estimate-then-optimize refugee matching method using three prediction model classes, then apply the evaluation methods described in Section~\ref{sec:sim_eval} to the proposed policies. In each case we train the estimation model on 33,000 refugees and evaluate the performance on 1,000 refugees with no prior placement restrictions (as done in \cite{bansak2018improving}). We perform each simulation 250 times to obtain histograms of performance estimates for each approach.

We compare the model-based and model-free methods described above for refugee assignments built using a variety of prediction model classes. First, we examine a LASSO-constrained logit regression over refugee covariates, locations, and their interactions (as done by~\cite{ahani2021placement}). This model class is very stable and generally performs well out of sample, but is improperly specified based on our causal model $f$. Second, we examine a family of Honest Random Forests (i.e. causal machine learning~\citep{wager2018estimation}); one for each location. These models are known for their ability to produce non-parametric unbiased estimates, but reduce the prediction power of the data in the process. Finally we examine a family of gradient-boosted classification trees (referred to as GBM); one for each location (as done by~\cite{bansak2018improving}). These models are very powerful, but not very stable and are prone to overfitting. 

For each model class our simulation:
\begin{enumerate}
    \item Generates a training dataset of $33,000$ refugees.
    \item Trains a prediction model $\hat{\beta}$ of the given model class on the training dataset.
    \item Generates a testing data set of $1,000$ refugees (all with case restriction = free).
    \item Optimizes the assignment of the refugees in the testing dataset according to $\hat{\beta}$, setting location capacities according to the observed number of matches in the testing dataset.
    \item Records evaluations of the proposed assignment $\pi^{\hat{\beta}}$ using $\hat{\beta}$ in the model-based method as well as by applying IPW to the testing dataset. 
    \item Generates 250 bootstrapped versions of the training dataset via resampling, trains a prediction model $\hat{\beta}^{(b)}$ of the given model class for each one, and records evaluation of the proposed assignment $\pi^{\hat{\beta}}$ using $\hat{\beta}^{(b)}$ in the model-based method. Additionally, we simulate 250 IPW evaluations on the same policy by shuffling the location assignments and regenerating employment outcomes. These results are presented in Figure~\ref{fig:boot_eval}.
    \item Repeats steps 3. to 5. until 250 estimate-then-optimize assignments have been evaluated. These results are presented in Figure~\ref{fig:dir_eval}.
\end{enumerate}

In addition to histograms which report estimates of the policy's performance over the 250 runs, we report ROC and calibration curves on a testing dataset for each prediction model in Appendix~\ref{apx:cal_ROC} to address J1. We generate these curves using SKLearn methods as we do with the logit models and gradient boosted classifiers. We generate Honest Random Forests using the EconML library. We implement the integer optimization problem in CVXPY using the SCIP solver. All simulations are performed in Python.

\subsection{Simulation Results}
\label{sec:sim_results}

Our simulations reveal that the model-based methods exhibit large biases from the winner's curse for all three prediction model classes, and that those biases are not eliminated when instead evaluated by models trained on bootstrapped versions of the training dataset. These simulations do not prove that the policies learned by~\cite{bansak2018improving},~\cite{ahani2021placement} are ineffective; rather, they show that this form of policy evaluation has a very high false positive rate (FPR)---consistently promising non-existent gains even when none are present---underscoring the need for valid policy evaluation methodologies in practice.

\begin{figure}
     \centering
     \begin{subfigure}[b]{0.31\textwidth}
         \centering
         \includegraphics[width=\textwidth]{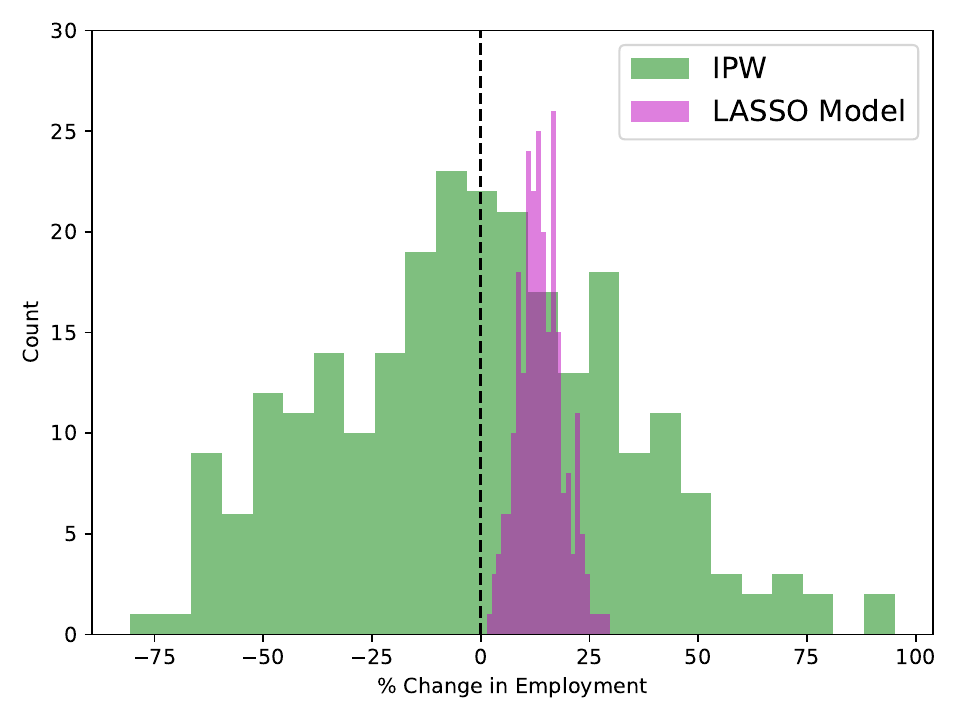}
         \caption{LASSO (Bias = 13\%)}
         \label{fig:dir_lasso}
     \end{subfigure}
     \hfill
     \begin{subfigure}[b]{0.31\textwidth}
         \centering
         \includegraphics[width=\textwidth]{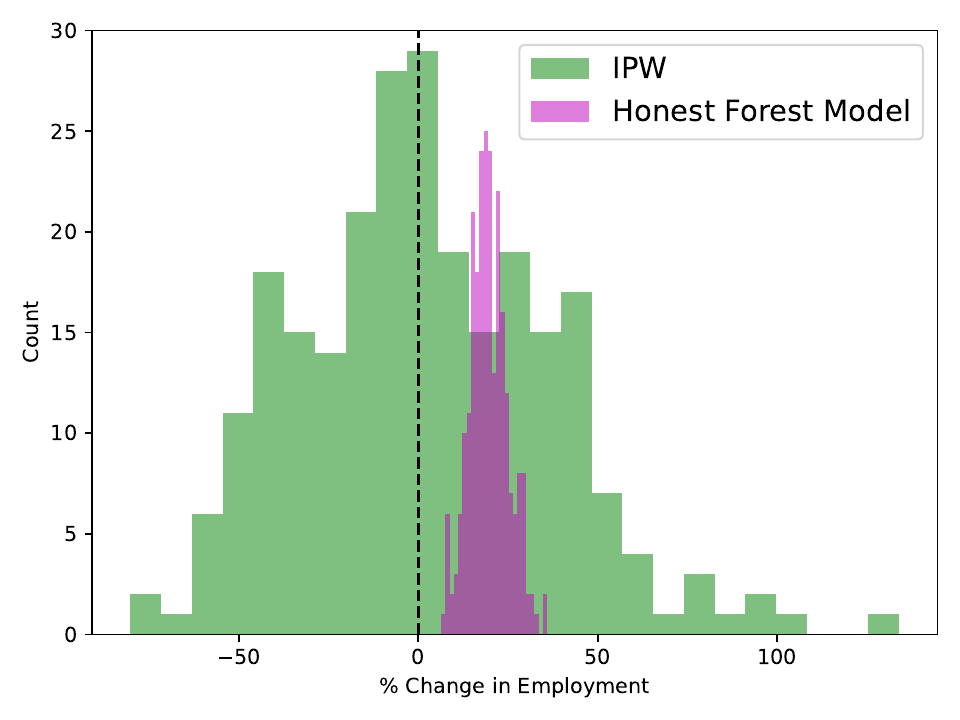}
         \caption{Honest RF (Bias = 20\%)}
         \label{fig:dir_hf}
     \end{subfigure}
     \hfill
     \begin{subfigure}[b]{0.31\textwidth}
         \centering
         \includegraphics[width=\textwidth]{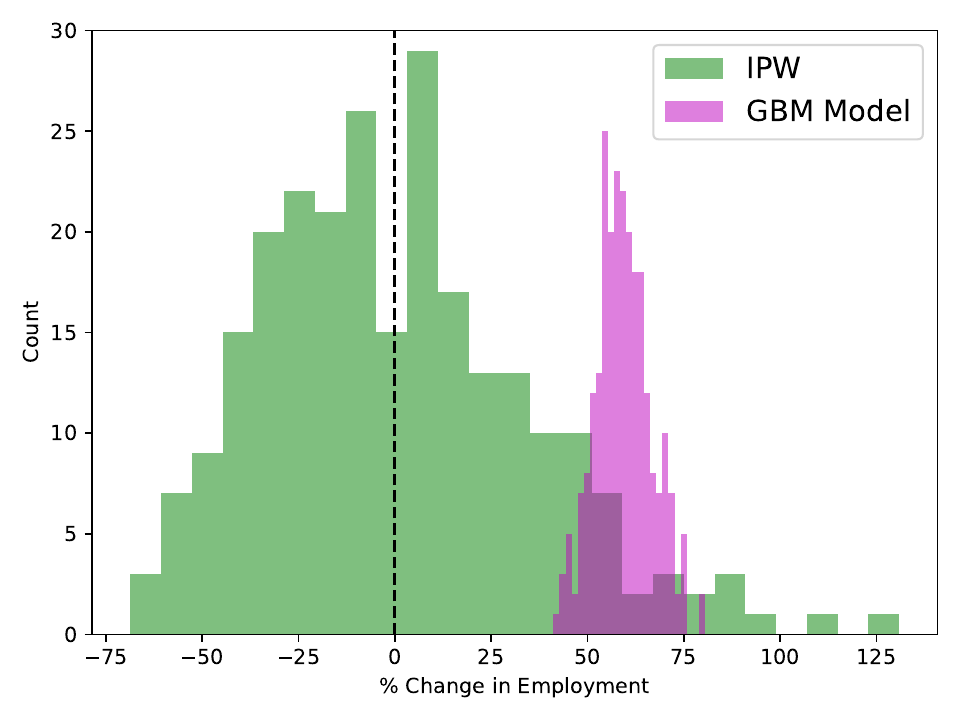}
         \caption{GBM (Bias = 59\%)}
         \label{fig:dir_gbm}
     \end{subfigure}
        \caption{Histograms comparing 250 evaluations according to the prediction model used to optimize the policy with IPW estimates. All prediction models exhibit biasess due to the winner's curse while IPW estimation suffers from excessive variance. All evaluations are reported in percent change in employment rate relative to the observed employment rate in the testing dataset.}
        \label{fig:dir_eval}
\end{figure}

Figure~\ref{fig:dir_eval} presents histograms comparing evaluations of the optimized policy's impact using both the prediction model informing the policy's construction and IPW estimation. Both LASSO and Honest Random Forests present a moderate winner's curse bias, while GBM exhibits an extreme bias (likely due to its lack of stability). On the other hand, while IPW estimation remains unbiased regardless of the prediction model used to optimize the policy, its variance, driven by both small sample size and large action space, precludes effective evaluation of the policy's impact. We caution readers from using these simulations to conclude that LASSO and Honest Random Forest provided less biased evaluation, as alternative parameterizations of our simulation lead to larger biases from LASSO and Honest Random Forest (compared to GBM). Instead, these simulations highlight the need to use provably valid policy optimization and evaluation methodologies that can limit variance in the presence of large action spaces and small sample sizes~\citep[see, e.g.,][]{bastani2025beating}.

\begin{figure}
     \centering
     \begin{subfigure}[b]{0.31\textwidth}
         \centering
         \includegraphics[width=\textwidth]{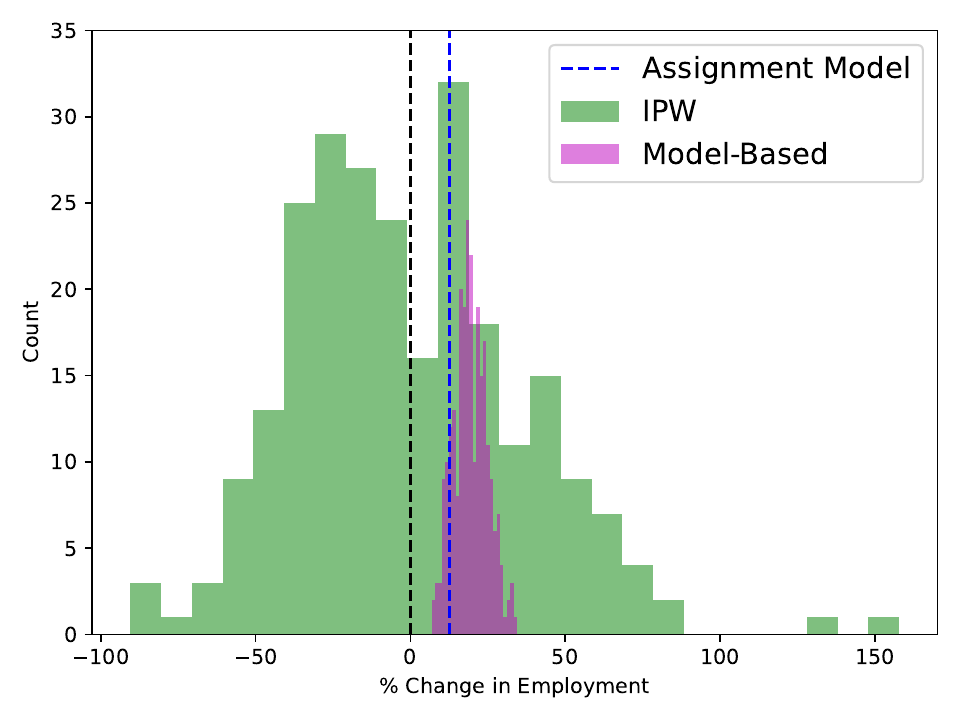}
         \caption{LASSO (Bias = 19\%)}
         \label{fig:boot_lasso}
     \end{subfigure}
     \hfill
     \begin{subfigure}[b]{0.31\textwidth}
         \centering
         \includegraphics[width=\textwidth]{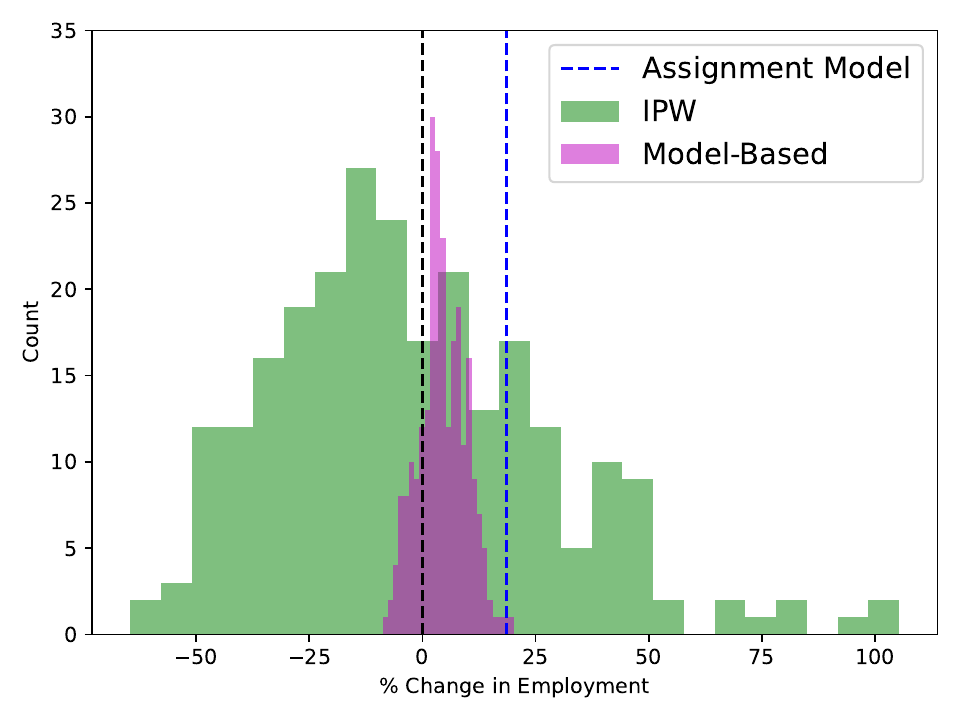}
         \caption{Honest RF (Bias = 4\%)}
         \label{boot:dir_hf}
     \end{subfigure}
     \hfill
     \begin{subfigure}[b]{0.31\textwidth}
         \centering
         \includegraphics[width=\textwidth]{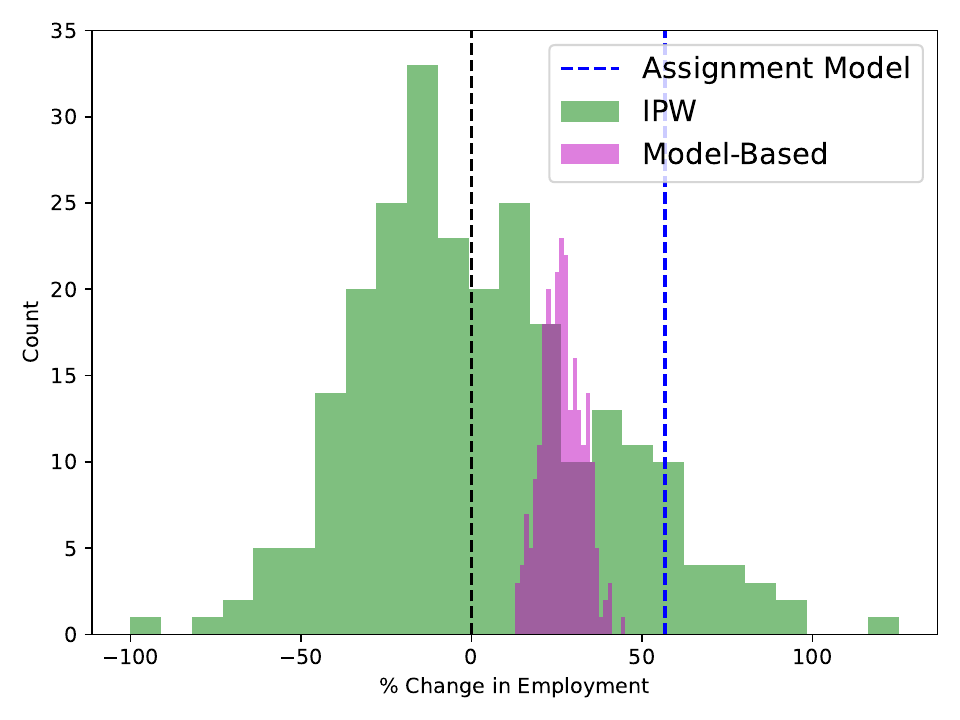}
         \caption{GBM (Bias = 27\%)}
         \label{fig:boot_gbm}
     \end{subfigure}
        \caption{Histograms comparing 250 evaluations according to prediction models built off bootstrapped samples of the training dataset with IPW estimates and direct evaluation by the prediction model which optimized assignments. The stability of the LASSO estimator leads the bootstrapped samples to exhibit a larger winner's curse than the original prediction model while GBM exhibits a reduced, but still quite large, winner's curse. Honest Random Forests were able to eliminate most of their winner's curse bias in this simulation, although in alternate setups we have observed sizable inaccuracies (both positive and negative) indicating the bootstrapping approach cannot generate trustworthy confidence intervals. All evaluations are reported in percent change in employment rate relative to the observed employment rate in the testing dataset.}
        \label{fig:boot_eval}
\end{figure}

Figure~\ref{fig:boot_eval} presents similar histograms, only now we compare evaluation of the optimized policy using models trained on bootstrapped versions of the training dataset, as performed by~\cite{ahani2021placement}. For LASSO, we find that the bootstrapped models perpetuate as much bias from the winner's curse as evaluating using the original model used for policy optimization. We believe this follows from misspecification of the stable LASSO model: it is consistently reproducing the predictions of the original model in its bootstraps due to the complicated impact of the refugees' covariates on employment outcomes. Meanwhile Honest Random Forest and GBM have no obvious misspecification and consequently are able to reduce, but not eliminate, the bias from their evaluations. Instead these techniques reproduce some of the winner's curse bias that the assignment model suffers due to correlation between the training dataset and its bootstraps. As such we would expect this bias to decrease as the power of the training set increases. However, it is unclear how the size of this bias relates to the size of the training dataset, making it very difficult to use the bootstrapping approach for evaluation.\footnote{We once again caution readers from using these simulations to conclude that Honest Random Forest is the ``best'' prediction model to use for bootstrapped policy evaluation as other simulation setups using Honest Random Forest models have retained large biases.}

Overall our simulations highlight that in order to confidently assess complex data-driven assignment policies, we need better policy evaluation methods. Without improved methods, the only way to confidently evaluate these policies is to retain large testing datasets to counteract the variance of IPW estimation, which is often economically infeasible.

\section{Discussion and Broader Implications} \label{sec:discussion}

Our results demonstrate that the winner’s curse is not merely a theoretical curiosity, nor is it confined to obviously misspecified models. Rather, it is a systematic failure mode of model-based policy evaluation in estimate–then–optimize pipelines. In our refugee-matching case study, state-of-the-art model-based methods produced large and stable---but entirely spurious---estimated gains, despite standard checks such as cross-validation and calibration. This is because standard diagnostics validate \textit{average} predictive performance, whereas policy optimization is designed to exploit precisely those parts of the input space where residual errors remain. To clarify why these spurious gains can survive familiar best practices, it is useful to contrast model-based and model-free evaluation across three data regimes.

The fundamental dilemma in evaluating data-driven policies is a bias–variance trade-off. Model-free methods such as inverse probability weighting (IPW) are asymptotically unbiased but suffer from prohibitively high variance when propensity scores are small or the action space is large. Model-based methods instead impute counterfactual outcomes, reducing variance by smoothing over noise. The problem, however, is that model-based estimates are biased in a systematically optimistic direction. The optimizer preferentially selects actions and subpopulations that appear best under the model---i.e., “wins” are disproportionately driven by favorable estimation error.

In “small data” settings, estimated models are noisy and unstable, discouraging any kind of policy evaluation. At the other extreme, in “large data” regimes with sufficient coverage, model-free estimators become precise, making it possible to accurately evaluate policy improvements.

The most problematic regime---and the one where the winner’s curse is most acute---is “medium data.” Here, the data are sufficient to estimate models that appear accurate, stable, and well-calibrated on average, yet insufficient to support low-variance model-free policy evaluation. In this regime, standard diagnostics create a false sense of security: a practitioner may observe high AUC, stability, and calibration, and conclude the model is a reliable proxy for reality. But policy optimization is effectively an adversarial search process that identifies and exploits residual model errors, especially in regions that are rare under the observed data distribution. Even a model with strong aggregate performance may overestimate outcomes for a small subset of cases; the optimizer will systematically select precisely those cases. 

Crucially, because model-free estimates remain high-variance in medium data settings---particularly with large or combinatorial action spaces---they often cannot deliver a decisive rebuttal to a precise (but biased) model-based estimate. The result is a predictable pattern: the optimistic model-based gain looks stable and persuasive, while unbiased checks are too noisy to serve as an effective counterweight. The refugee matching case study exemplifies this regime: roughly 30,000 observations are sufficient to support seemingly reasonable predictive models, yet the sparsity of the assignment space ensures that model-free evaluation is too noisy to robustly validate the optimizer’s purported gains.

Because the medium-data regime is pervasive, the winner’s curse is a systemic threat. Many high-stakes decision problems possess enough data to estimate reasonable statistical models, but not enough to rigorously validate the tails of the distribution where optimized policies concentrate. Thus, we argue against using model-based imputation for policy evaluation, even when the estimated models appear strong by conventional metrics. Effectively doing so requires rigorously integrating model-based data with real-world outcomes~\citep{mandyam_perry:_2025}.

Rather, emerging research addresses the challenge of the high variance of model-free evaluation by rethinking the learning process itself; for instance,~\cite{chernozhukov2025policy,bastani2025beating} propose algorithms that seek statistically guaranteed improvements rather than simple expected value maximization, while~\cite{banerjee2025selecting} demonstrate that constraining the complexity of the policy class can effectively curb the winner's curse.

\section*{Acknowledgments}
The authors are grateful to Gad Allon, Jackie Baek, Kirk Bansak, Mohsen Bayati, Ron Berman, Rob Bray, Gerard Cachon, Vishal Gupta, Dean Knox, Elisabeth Paulson, Neha Sharma, Yannis Stamatopoulos, Alexander Teytelboym, Stefan Wager, and others for helpful feedback. This research was supported by generous funding from the Wharton AI \& Analytics Initiative.

\bibliographystyle{plainnat}  
\bibliography{refs}  
\ECSwitch
\ECHead{Appendix}

\section{Proofs for Section~\ref{sec:regularized}}
\label{sec:prop:regularized:proof}

We begin by proving a number of general results on accuracy and stability of ridge regression; our proof of Proposition~\ref{prop:regularized} relies on the general theory to establish accuracy and stability. Our proof that the winner's curse happens is specific to the problem construction established in Section~\ref{sec:regularized}.

\subsection{General Problem Formulation}

We begin by formalizing the problem of policy evaluation where the counterfactual outcomes are estimated using a ridge regression model. We consider a covariate space $\mathcal{X}\subseteq\mathbb{R}^d$, a compact space of treatments $\mathcal{T}$, a parameter space $\beta\in\mathcal{B}\subseteq\mathbb{R}^m$, a feature space $\mathcal{Z}=\mathbb{R}^m$, and a feature map $\phi:\mathcal{X}\times\mathcal{T}\to\mathcal{Z}$. The \emph{decision objective} is $f^{\beta}(x,t)=\beta^\top\phi(x,t)$, and the \emph{$\beta$-optimal policy} is $\pi^{\beta}(x)=\operatorname*{\arg\max}_{t\in\mathcal{T}}f^{\beta}(x,t)$. We consider a probability measure $\mathbb{P}_x$ (capturing the covariate distribution) and $\mathbb{P}_t$ (capturing the treatment distribution in the historical data), defining the product measure $\mathbb{P}_{x,t}=\mathbb{P}_x\times\mathbb{P}_t$. We assume $\phi$ is measurable and bounded in expectation: $\mathbb{E}_x\left[\max_{t\in\mathcal{T}}\|\phi(x,t)\|_2\right]\le\phi_{\text{max}}$, and we assume that $\phi(x,t)$ is $\eta$-subgaussian. We let $\beta^*\in\mathcal{B}$ denote the parameters for the true decision objective. We evaluate parameter estimates $\beta\in\mathcal{B}$ using two metrics: (1) global prediction accuracy, and (2) the optimistic bias of the resulting policy $\pi^{\beta}$.
\begin{definition}
\rm
Given $\beta\in\mathcal{B}$, the \emph{excess mean-squared error (MSE)} is the generalization error relative to the true model (typically estimated on a held-out test set):
\begin{align*}
\mathcal{E}(\beta)=\mathbb{E}_{x,t}[(f^{\beta}(x,t)-f^{\beta^*}(x,t))^2].
\end{align*}
\end{definition}
\begin{definition}
\rm
Given $\beta,\beta'\in\mathcal{B}$, the \emph{optimism bias} of $\beta$ when evaluated using $\beta'$ is
\begin{align*}
\mathcal{R}(\beta;\beta')
=\mathcal{P}(\beta;\beta')-\mathcal{P}(\beta;\beta^*)
\qquad\text{where}\qquad
\mathcal{P}(\beta;\beta')=\mathbb{E}_x[f^{\beta'}(x,\pi^{\beta}(x))].
\end{align*}
\end{definition}
To estimate $\beta^*$, we consider a historical dataset $W=\{(z_i,y_i)\}_{i=1}^n$, where $y_i=f^{\beta^*}(x_i,t_i)+\epsilon_i$ with i.i.d. noise $\epsilon_i \sim \mathcal{N}(0,\sigma^2)$; we let $\mathbb{P}_{x,t,\epsilon}=\mathbb{P}_{x,t}\times\mathbb{P}_{\epsilon}$ denote the distribution of a single historical example. We employ the classical ridge regression estimator, which serves as a proxy for any stable, regularized machine learning method:
\begin{align*}
\hat\beta^{(\lambda)}=\operatorname*{\arg\min}_{\beta\in\mathbb{R}^m}\left\{\frac{1}{n}\|Z\beta-Y\|_2^2+\lambda\|\beta\|^2\right\},
\end{align*}
where $\lambda\in\mathbb{R}_{>0}$ is the regularization hyperparameter, and
\begin{align*}
Z=\begin{bmatrix}
\horzbar & z_1^\top & \horzbar \\
\vdots & \vdots & \vdots \\
\horzbar & z_n^\top & \horzbar
\end{bmatrix},
\qquad\qquad
Y=\begin{bmatrix}
y_1 \\
\vdots \\
y_n
\end{bmatrix}.
\end{align*}
Let $\Sigma=\mathbb{E}_{x,t}[zz^\top]$ be the true covariance matrix and $\hat\Sigma=n^{-1}Z^\top Z$ be the empirical covariance matrix, and let $\Sigma_{\lambda}=\Sigma+\lambda I$ and $\hat\Sigma_{\lambda}=\hat\Sigma+\lambda I$, where $I\in\mathbb{R}^{m\times m}$ is the identity matrix. Then,
\begin{align*}
\hat\beta^{(\lambda)}
&=(n\hat\Sigma_{\lambda})^{-1}Z^\top Y
=(I-\lambda\hat\Sigma_{\lambda}^{-1})\beta^*+\frac{1}{n}\hat\Sigma_{\lambda}^{-1}Z^\top E,
\end{align*}
where $E=\begin{bmatrix}\epsilon_1&\cdots\epsilon_n\end{bmatrix}^\top$ is the vector of noise terms. Given $\Delta\in\mathbb{R}_{>0}$, define the event
\begin{align*}
E_{\Delta}=\mathbbm{1}(\|\hat\Delta\|_2\le\Delta)
\qquad\text{where}\qquad
\hat\Delta=\hat\Sigma-\Sigma.
\end{align*}
In addition, given $\zeta\in\mathbb{R}_{>0}$, define the event
\begin{align*}
E_{\zeta}'=\mathbbm{1}\left(\|\hat\Sigma_{\lambda}^{-1}Z^\top E\|_2^2\le\frac{5nm\sigma^2\zeta}{\lambda}\wedge\|\hat\Sigma_{\lambda}^{-1/2}Z^\top E\|_2^2\le5nm\sigma^2\zeta\right).
\end{align*}
We let $\mathbb{P}_E=(\mathbb{P}_{\epsilon})^n$ denote the measure of $E$, and $\mathbb{P}_Z=(\phi_*\mathbb{P}_{x,t})^n$ of $Z$ (where $f_*\beta$ denotes the pushforward measure).

\subsection{Results for the General Case}

First, we prove that $E_{\Delta}$ holds with high probability as $n$ grows large, considering general subexponential features; we do so in terms of constants in the Bernstein moment condition, which we will verify for our specific problem.
\begin{lemma}
\label{lem:eventprobability1}
Suppose there exist $v,B\in\mathbb{R}_{>0}$ satisfying $\mathbb{E}_{x,t}[|z_jz_k-\Sigma_{jk}|^r]\le r!vB^{r-2}/2$ for all $j,k\in[m]$ and integers $r\ge2$. Then, for every $\Delta>0$, we have
\begin{align*}
\mathbb{P}_Z[E_{\Delta}]\ge1-2m^2\exp\left(-\frac{n\Delta^2}{2m^2v+2mB\Delta}\right).
\end{align*}
\end{lemma}
\begin{proof}{Proof.}
For fixed $j,k\in[m]$, $z_{ij}z_{ik}-\Sigma_{jk}$ (for $i\in[n]$) is independent and has mean zero. Thus, by Bernstein's inequality~\citep[Corollary~2.11]{boucheron2013concentration}, we have
\begin{align*}
\mathbb{P}_Z[|\hat\Delta_{jk}|\ge\Delta/m]
\le2\exp\left(-\frac{n(\Delta/m)^2}{2(v+B\Delta/m)}\right)
=2\exp\left(-\frac{n\Delta^2}{2m^2v+2mB\Delta}\right).
\end{align*}
Since $\|\hat\Delta\|_2\le m\max_{j,k\in[m]}|\hat\Delta_{jk}|$, the claim follows by a union bound over all $m^2$ entries. \qed
\end{proof}
Next, we prove a high-probability bound for $E_{\zeta}'$.
\begin{lemma}
\label{lem:eventprobability2}
For $\zeta\ge1$, we have
\begin{align*}
\mathbb{P}_W[E_{\zeta}']\ge1-e^{-\zeta}.
\end{align*}
\end{lemma}
\begin{proof}{Proof.}
Let $Z/\sqrt{n}=U\hat\Lambda V^\top$ be the singular value decomposition of $Z/\sqrt{n}$, where
\begin{align*}
\hat\Lambda=\begin{bmatrix}
\tilde\Lambda \\
0_{(n-m)\times m}
\end{bmatrix}
\qquad\text{where}\qquad
\tilde\Lambda=\begin{bmatrix}
\lambda_1 & 0 & \cdots & 0 \\
0 & \lambda_2 & \cdots & 0 \\
\vdots & \vdots & \ddots & \vdots \\
0 & 0 & \cdots & \lambda_m
\end{bmatrix}
\end{align*}
for some $\lambda_1,...,\lambda_m\in\mathbb{R}_{\ge0}$, and $U\in\mathbb{R}^{n\times n}$ and $V\in\mathbb{R}^{m\times m}$ are orthogonal. Then, we have
\begin{align*}
\hat\Sigma_{\lambda}
=\frac{1}{n}Z^\top Z+\lambda I
=V\hat\Lambda^\top\hat\Lambda V^\top+\lambda I
=V(\tilde\Lambda^2+\lambda I)V^\top.
\end{align*}
Thus,
\begin{align*}
\hat\Sigma_{\lambda}^{-1}Z^\top E
=\sqrt{n}V(\tilde\Lambda^2+\lambda I)^{-1}\hat\Lambda^\top U^\top E.
\end{align*}
Let $\tilde{E}=\begin{bmatrix}\tilde\epsilon_1&\cdots\tilde\epsilon_m\end{bmatrix}^\top$ with $\tilde{\epsilon}_i=(U^\top E)_i$ for $i\in[m]$; note that since $U$ is orthogonal, $U^\top E\sim\mathcal{N}(0,\sigma^2I_n)$ has the same distribution as $E$. Furthermore, we have
\begin{align*}
\hat\Sigma_{\lambda}^{-1}Z^\top E=\sqrt{n}V\tilde\Lambda'\tilde{E}
\qquad\text{where}\qquad
\tilde\Lambda'=\begin{bmatrix}
\frac{\lambda_1}{\lambda_1^2+\lambda} & 0 & \cdots & 0 \\
0 & \frac{\lambda_2}{\lambda_2^2+\lambda} & \cdots & 0 \\
\vdots & \vdots & \ddots & \vdots \\
0 & 0 & \cdots & \frac{\lambda_m}{\lambda_m^2+\lambda}
\end{bmatrix}.
\end{align*}
Now, note that
\begin{align*}
\|\hat\Sigma_{\lambda}^{-1}Z^\top E\|_2^2
=n\sum_{i=1}^m\frac{\lambda_i^2}{(\lambda_i^2+\lambda)^2}\cdot\tilde{\epsilon}_i^2
=n\sum_{i=1}^m\left(\frac{\lambda_i^2+\lambda}{(\lambda_i^2+\lambda)^2}-\frac{\lambda}{(\lambda_i^2+\lambda)^2}\right)\cdot\tilde{\epsilon}_i^2
\le\frac{n}{\lambda}\sum_{i=1}^m\tilde\epsilon_i^2=\frac{n\sigma^2}{\lambda}\sum_{i=1}^m\frac{\tilde\epsilon_i^2}{\sigma^2}.
\end{align*}
Similarly, we have
\begin{align*}
\hat\Sigma_{\lambda}^{-1/2}Z^\top E
=\sqrt{n}V(\tilde\Lambda^2+\lambda I)^{-1/2}\hat\Lambda^\top U^\top E
=\sqrt{n}V\tilde\Lambda''\tilde{E}
\qquad\text{where}\qquad
\tilde\Lambda''=\begin{bmatrix}
\frac{\lambda_1}{\sqrt{\lambda_1^2+\lambda}} & 0 & \cdots & 0 \\
0 & \frac{\lambda_2}{\sqrt{\lambda_2^2+\lambda}} & \cdots & 0 \\
\vdots & \vdots & \ddots & \vdots \\
0 & 0 & \cdots & \frac{\lambda_m}{\sqrt{\lambda_m^2+\lambda}}
\end{bmatrix}.
\end{align*}
Now, note that
\begin{align*}
\|\hat\Sigma_{\lambda}^{-1/2}Z^\top E\|_2^2
=n\sum_{i=1}^m\frac{\lambda_i^2}{\lambda_i^2+\lambda}\cdot\tilde{\epsilon}_i^2
\le n\sum_{i=1}^m\tilde\epsilon_i^2=n\sigma^2\sum_{i=1}^m\frac{\tilde\epsilon_i^2}{\sigma^2}.
\end{align*}
Finally, by the Laurent-Massart inequality, we have
\begin{align*}
\mathbb{P}_{\tilde{E}}\left[\sum_{i=1}^m\frac{\tilde\epsilon_i^2}{\sigma^2}\le m+2m\sqrt{\zeta}+2\zeta\right]\ge1-e^{-\zeta},
\end{align*}
so the claim follows. \qed
\end{proof}
Next, we prove a standard formula for the excess MSE of any parameter vector $\beta$.
\begin{lemma}
\label{lem:excessmse}
We have $\mathcal{E}(\beta)=\|\Sigma^{1/2}(\beta-\beta^*)\|_2^2$.
\end{lemma}
\begin{proof}{Proof.}
Note that
\begin{align*}
\mathcal{E}(\beta)&=\mathbb{E}_{x,t,\epsilon}[(\beta^\top z-y)^2]-\mathbb{E}[(\beta^{*\top}z-y)^2] \\
&=\mathbb{E}_{x,t,\epsilon}[(\beta^\top z-\beta^{*\top}z-\epsilon)^2]-\mathbb{E}_{x,t,\epsilon}[\epsilon^2] \\
&=\mathbb{E}_{x,t,\epsilon}[((\beta-\beta^*)^\top z)^2] \\
&=\mathbb{E}_{x,t,\epsilon}[\text{tr}((\beta-\beta^*)(\beta-\beta^*)^\top zz^\top)] \\
&=\text{tr}((\beta-\beta^*)(\beta-\beta^*)^\top\Sigma) \\
&=\|\Sigma^{1/2}(\beta-\beta^*)\|_2^2,
\end{align*}
as claimed. \qed
\end{proof}
Our next result establishes a straightforward bound on the sensitivity of the estimated policy improvement $\mathcal{P}(\beta;\beta')$ to $\beta'$ (note, however that $\mathcal{P}$ can be arbitrarily sensitive to $\beta$).
\begin{lemma}
\label{lem:performancestability}
Given $\beta,\beta_1,\beta_2\in\mathcal{B}$, if $\|\beta_1-\beta_2\|_2\le\epsilon$, then $|\mathcal{P}(\beta;\beta_1)-\mathcal{P}(\beta;\beta_2)|\le\epsilon\phi_{\text{max}}$.
\end{lemma}
\begin{proof}{Proof.}
Note that
\begin{align*}
|\mathcal{P}(\beta;\beta_1)-\mathcal{P}(\beta;\beta_2)|
&=|\mathbb{E}_x[(\beta_1-\beta_2)^\top\phi(x,\pi^{\beta}(x))]| \\
&\le\|\beta_1-\beta_2\|_2\cdot\mathbb{E}_x[\|\phi(x,\pi^{\beta}(x))\|_2] \\
&\le\epsilon\phi_{\text{max}},
\end{align*}
as claimed. \qed
\end{proof}
Another simple result bounds the maximum eigenvalue of $\hat\Sigma_{\lambda}^{-1}$.
\begin{lemma}
\label{lem:sigmahatinv}
We have $\|\hat\Sigma_{\lambda}^{-1}\|_2\le\lambda^{-1}$.
\end{lemma}
\begin{proof}{Proof.}
Note that
\begin{align*}
\|\hat\Sigma_{\lambda}^{-1}\|_2
=\|(\hat\Sigma+\lambda I)^{-1}\|_2
=\frac{1}{\lambda_{\text{min}}(\hat\Sigma+\lambda I)}
=\frac{1}{\lambda+\lambda_{\text{min}}(\hat\Sigma)}
\le\frac{1}{\lambda},
\end{align*}
where the last step follows since $\hat\Sigma$ is positive semi-definite. \qed
\end{proof}
Next, we bound the symmetric matrix $\hat\Sigma_{\lambda}^{-1/2}\Sigma\hat\Sigma_{\lambda}^{-1/2}$; this bound is a key part of our analysis of the excess MSE of $\hat\beta^{(\lambda)}$.
\begin{lemma}
\label{lem:sigmahatinvsigma}
On event $E_{\Delta}$ with $\Delta\le\lambda$, we have $\|\hat\Sigma_{\lambda}^{-1/2}\Sigma\hat\Sigma_{\lambda}^{-1/2}\|_2\le1$.
\end{lemma}
\begin{proof}{Proof.}
On event $E_{\Delta}$, we have $\hat\Delta\succeq-\Delta I$. Hence, since $\Delta\le\lambda$,
\begin{align*}
\hat\Sigma_{\lambda}-\Sigma
=\hat\Delta+\lambda I
\succeq(\lambda-\Delta)I
\succeq0.
\end{align*}
Since $\Sigma\succeq0$ and $\hat\Sigma_{\lambda}\succ0$, multiplying on both sides by $\hat\Sigma_{\lambda}^{-1/2}$, we have
\begin{align*}
\hat\Sigma_{\lambda}^{-1/2}\Sigma\hat\Sigma_{\lambda}^{-1/2}\preceq I.
\end{align*}
Also, note that $\hat\Sigma_{\lambda}^{-1/2}\Sigma\hat\Sigma_{\lambda}^{-1/2}\succeq0$. As a consequence, this matrix is symmetric positive semi-definite with all eigenvalues at most one, so the claim follows. \qed
\end{proof}
Now, we state and prove our first main result, which bounds the accuracy of $\hat\beta^{(\lambda)}$ on events $E_{\Delta}$ and $E_{\zeta}'$. Intuitively, the first term in the bound is the bias term (which becomes small as $\lambda\to0$), and the second term is the variance term (which becomes small as $n\to\infty$). We generally think of $\lambda$ as scaling as $1/\sqrt{n}$.
\begin{proposition}[Accuracy]
\label{prop:accuracy}
On event $E_{\Delta}$ with $\Delta\le\lambda$ and $E_{\zeta}'$ with $\zeta\ge1$, we have
\begin{align*}
\mathcal{E}(\hat\beta^{(\lambda)})\le2\lambda\cdot\|\beta^*\|_2^2+\frac{10m\sigma^2\zeta}{n}.
\end{align*}
\end{proposition}
\begin{proof}{Proof.}
By Lemma~\ref{lem:sigmahatinvsigma}, we have
\begin{align*}
\|\Sigma^{1/2}\hat\Sigma_{\lambda}^{-1/2}\|_2^2
=\|\hat\Sigma_{\lambda}^{-1/2}\Sigma\hat\Sigma_{\lambda}^{-1/2}\|_2
\le1.
\end{align*}
Then,  note that
\begin{align*}
\|\Sigma^{1/2}(\hat\beta^{(\lambda)}-\beta^*)\|_2
&=\left\|\Sigma^{1/2}\hat\Sigma_{\lambda}^{-1/2}\left(-\lambda\hat\Sigma_{\lambda}^{-1/2}\beta^*+\frac{1}{n}\hat\Sigma_{\lambda}^{-1/2}Z^\top E\right)\right\|_2 \\
&\le\lambda\|\hat\Sigma_{\lambda}^{-1/2}\|_2\cdot\|\beta^*\|_2+\frac{1}{n}\|\hat\Sigma_{\lambda}^{-1/2}Z^\top E\|_2 \\
&\le\sqrt{\lambda}\cdot\|\beta^*\|_2+\sqrt{\frac{5m\sigma^2\zeta}{n}},
\end{align*}
where the last inequality follows by Lemma~\ref{lem:sigmahatinv} and event $E_{\zeta}'$. By Lemma~\ref{lem:excessmse}, the square of the left-hand side is $\mathcal{E}(\hat\beta^{(\lambda)})$, so the claim follows from the inequality $(a+b)^2\le2(a^2+b^2)$. \qed
\end{proof}
Our second major result provides a stability guarantee for $\hat\beta^{(\lambda)}$; in particular, it says $\hat\beta^{(\lambda)}$ concentrates around a deterministic quantity $\bar\beta^{(\lambda)}$ on events $E_{\Delta}$ and $E_{\zeta}'$ (which, by Lemmas~\ref{lem:eventprobability1} \&~\ref{lem:eventprobability2}, holds with high probability over $W$). This result straightforwardly implies that all $\hat\beta^{(\lambda)}$ are pairwise close together. One thing to note in this result is that both $\Delta$ and $\lambda$ scale as $1/\sqrt{n}$; thus, to obtain a bound that goes to zero, we need $\lambda/\Delta$ to be sufficiently large.
\begin{proposition}[Stability]
\label{prop:betahatstability}
Letting $\bar\beta^{(\lambda)}
=(I-\lambda\Sigma_{\lambda}^{-1})\beta^*$, on event $E_{\Delta}$ with $\Delta\le\lambda/4$ and event $E_{\zeta}'$ with $\zeta\in\mathbb{R}_{>0}$, we have
\begin{align*}
\|\hat\beta^{(\lambda)}-\bar\beta^{(\lambda)}\|_2^2\le\frac{2\Delta^2\|\beta^*\|_2^2}{\lambda^2}
+\frac{10m\sigma^2\zeta}{n\lambda}.
\end{align*}
\end{proposition}
\begin{proof}{Proof.}
We prove a slightly more general result, where we let $\bar\beta^{(\lambda)}=(I-\lambda(\Sigma+\Gamma+\lambda I)^{-1})\beta^*$ for some symmetric $\Gamma\in\mathbb{R}^{m\times m}$ satisfying $\Sigma+\Gamma\succeq0$ and $\gamma=\|\Gamma\|_2\le\lambda/4$. Then, letting $\hat\Gamma=\hat\Delta-\Gamma$ (so $\|\hat\Gamma\|_2\le\gamma+\Delta$), we show that
\begin{align*}
\|\hat\beta^{(\lambda)}-\bar\beta^{(\lambda)}\|_2^2\le\frac{2(\Delta+\gamma)^2\|\beta^*\|_2^2}{\lambda^2}
+\frac{10m\sigma^2\zeta}{n\lambda}.
\end{align*}
The original result follows from taking $\Gamma=0$, so $\bar\beta^{(\lambda)}$ is unchanged and $\gamma=0$. Now, note that
\begin{align*}
\hat\Sigma_{\lambda}^{-1}-(\Sigma+\Gamma+\lambda I)^{-1}
=-\hat\Sigma_{\lambda}^{-1}\hat\Gamma(\Sigma+\Gamma+\lambda I)^{-1}.
\end{align*}
Since $\hat\Sigma\succeq0$ and $\Sigma+\Gamma\succeq0$, both inverse matrices have operator norm at most $1/\lambda$, so
\begin{align*}
\|\hat\Sigma_{\lambda}^{-1}-(\Sigma+\Gamma+\lambda I)^{-1}\|_2
\le\|\hat\Sigma_{\lambda}^{-1}\|_2\cdot\|\hat\Gamma\|_2\cdot\|(\Sigma+\Gamma+\lambda I)^{-1}\|_2
\le\frac{\Delta+\gamma}{\lambda^2}.
\end{align*}
Now, by the inequality $\|u+v\|_2^2\le2\|u\|_2^2+2\|v\|_2^2$, we have
\begin{align*}
\|\hat\beta^{(\lambda)}-\bar\beta^{(\lambda)}\|_2^2
&=\left\|-\lambda\left(\hat\Sigma_\lambda^{-1}-(\Sigma+\Gamma+\lambda I)^{-1}\right)\beta^*+\frac{1}{n}\hat\Sigma_\lambda^{-1}Z^\top E\right\|_2^2 \\
&\le2\lambda^2\|\beta^*\|_2^2\cdot\|\hat\Sigma_{\lambda}^{-1}-(\Sigma+\Gamma+\lambda I)^{-1}\|_2^2
+\frac{2}{n^2}\|\hat\Sigma_{\lambda}^{-1}Z^\top E\|_2^2 \\
&\le\frac{2(\Delta+\gamma)^2\|\beta^*\|_2^2}{\lambda^2}+\frac{10m\sigma^2\zeta}{n\lambda},
\end{align*}
where the last step follows from event $E_{\zeta}'$. \qed
\end{proof}
Our next result is a straightforward consequence of our previous result---it says that since $\hat\beta^{(\lambda)}$ concentrates to $\bar\beta^{(\lambda)}$, then the optimism bias when evaluating under $\hat\beta^{(\lambda)}$ correspondingly concentrates to the optimism bias under $\bar\beta^{(\lambda)}$.
\begin{proposition}
\label{prop:biasstability}
For any $\beta\in\mathbb{R}^m$, letting $\bar\beta^{(\lambda)}=(I-\lambda\Sigma_{\lambda}^{-1})\beta^*$, on event $E_{\Delta}$ for some $\Delta\le\lambda/4$ and on event $E_{\zeta}'$ for some $\zeta\in\mathbb{R}_{>0}$, we have
\begin{align*}
|\mathcal{R}(\beta;\hat\beta^{(\lambda)})-\mathcal{R}(\beta;\bar\beta^{(\lambda)})|
=|\mathcal{P}(\beta;\hat\beta^{(\lambda)})-\mathcal{P}(\beta;\bar\beta^{(\lambda)})|
\le\phi_{\text{max}}\cdot\sqrt{\frac{2\Delta^2\|\beta^*\|_2^2}{\lambda^2}+\frac{10m\sigma^2\zeta}{n\lambda}}.
\end{align*}
\end{proposition}
\begin{proof}{Proof.}
We prove a slightly more general result, where we let $\bar\beta^{(\lambda)}=(I-\lambda(\Sigma+\Gamma+\lambda I)^{-1})\beta^*$ for some symmetric $\Gamma\in\mathbb{R}^{m\times m}$ satisfying $\Sigma+\Gamma\succeq0$ and $\gamma=\|\Gamma\|_2\le\lambda/4$. Then, letting $\hat\Gamma=\hat\Delta-\Gamma$ (so $\|\hat\Gamma\|_2\le\Delta+\gamma$), we show that
\begin{align*}
|\mathcal{R}(\beta;\hat\beta^{(\lambda)})-\mathcal{R}(\beta;\bar\beta^{(\lambda)})|
=|\mathcal{P}(\beta;\hat\beta^{(\lambda)})-\mathcal{P}(\beta;\bar\beta^{(\lambda)})|
\le\phi_{\text{max}}\cdot\sqrt{\frac{2(\Delta+\gamma)^2\|\beta^*\|_2^2}{\lambda^2}+\frac{10m\sigma^2\zeta}{n\lambda}}.
\end{align*}
The original result follows from taking $\Gamma=0$, so $\bar\beta^{(\lambda)}$ is unchanged and $\gamma=0$. Now, note that
\begin{align*}
|\mathcal{R}(\beta;\hat\beta^{(\lambda)})-\mathcal{R}(\beta;\bar\beta^{(\lambda)})|
=|\mathcal{P}(\beta;\hat\beta^{(\lambda)})-\mathcal{P}(\beta;\bar\beta^{(\lambda)})|
\le\|\hat\beta^{(\lambda)}-\bar\beta^{(\lambda)}\|_2\cdot\phi_{\text{max}},
\end{align*}
where the inequality follows by Lemma~\ref{lem:performancestability}. Then, the result follows from Proposition~\ref{prop:betahatstability}.
\end{proof}

\subsection{Proof of Proposition~\ref{prop:regularized}}
\label{sec:prop:regularized:proof2}

Let $\Delta=\sqrt{120\log(144/\delta)/n}$ (so $\lambda=\alpha\Delta$) and $\zeta=\log(144/\delta)$. Our assumptions $\lambda\le1$ and $\alpha\ge32$ imply $0<\Delta\le1/32<1$, so $p=1-\Delta/2\in(0,1)$. Note that $\|\beta^*\|_2=b$, and
\begin{align*}
\phi_{\text{max}}
=\mathbb{E}_x\left[\max_{t\in\mathcal{T}}\sqrt{x^2+t^2+(xt)^2}\right]
=\mathbb{E}_x[\sqrt{2x^2+1}]
\le\sqrt{2}\cdot\mathbb{E}_x[|x|]+1
=\frac{2}{\sqrt{\pi}}+1
\le3.
\end{align*}
Next, we verify the moment condition in Lemma~\ref{lem:eventprobability1} with $v=6$ and $B=2$. For $z=\phi(x,t)=(t,x,tx)^\top$, every product $z_jz_k$ is one of $t$, $tx$, $x^2$, or $tx^2$. For any $u\in\{t,x^2,tx^2\}$, let $u'$ be an independent copy. Since $u,u'\ge0$, for every integer $r\ge2$, we have
\begin{align*}
\mathbb{E}_{u,u'}[|u-\mathbb{E}[u]|^r]
\le\mathbb{E}_{u,u'}[|u-u'|^r]
\le2\mathbb{E}_u[u^r]
\le2\mathbb{E}_x[x^{2r}]
=2(2r-1)!!
\le3r!2^{r-2},
\end{align*}
where the first inequality follows from Jensen's inequality, $k!!$ is the double factorial of $k$, and the last inequality follows by induction on $r$. For $u=tx$, we have $\mathbb{E}[u]=0$ and
\begin{align*}
\mathbb{E}[|u|^r]\le\mathbb{E}[|x|^r]
\le\sqrt{\mathbb{E}[x^{2r}]}
\le(2r-1)!!
\le3r!\,2^{r-2}.
\end{align*}
Thus, the moment condition holds. By Lemma~\ref{lem:eventprobability1} with $m=3$, $v=6$, and $B=2$, we have
\begin{align*}
\mathbb{P}_Z[E_{\Delta}^c]
\le18\exp\left(-\frac{n\Delta^2}{108+12\Delta}\right)
\le18\exp\left(-\frac{n\Delta^2}{120}\right)
=18e^{-\log(144/\delta)}
=\frac{\delta}{8}.
\end{align*}
Also, $\zeta=\log(144/\delta)\ge1$, so by Lemma~\ref{lem:eventprobability2}, we have $\mathbb{P}_W[(E_{\zeta}')^c]\le e^{-\zeta}=\delta/144\le\delta/8$. Taking a union bound over the two events for each of the four models, then all these events hold simultaneously with probability at least $1-\delta$.

Next, we prove each of our main results conditioned on these events holding.

\paragraph{Accuracy.}

By Proposition~\ref{prop:accuracy}, on event $E_{\Delta}$ and $E_{\zeta}'$, we have
\begin{align*}
\mathcal{E}(\hat\beta^{(\lambda)})
&\le2\lambda\cdot\|\beta^*\|_2^2+\frac{10m\sigma^2\zeta}{n} \\
&=2b^2\alpha\sqrt{\frac{120\log(144/\delta)}{n}}+\frac{30\log(144/\delta)}{n}.
\end{align*}

\paragraph{Stability.}

Let
\begin{align*}
\bar{\Sigma}=\begin{bmatrix}
1 & 0 & 0 \\
0 & 1 & 1 \\
0 & 1 & 1
\end{bmatrix}
\qquad\text{and}\qquad
\Gamma=\begin{bmatrix}
1-p & 0 & 0 \\
0 & 0 & 1-p \\
0 & 1-p & 1-p
\end{bmatrix},
\end{align*}
so $\bar\Sigma=\Sigma+\Gamma$. Further define $\bar\Sigma_{\lambda}=\bar\Sigma+\lambda I$. The matrix $\Gamma$ is symmetric, and $\Sigma+\Gamma=\bar\Sigma\succeq0$. Also, note that
\begin{align*}
\gamma=\|\Gamma\|_2
=\frac{1+\sqrt{5}}{2}(1-p)
=\frac{1+\sqrt{5}}{4}\Delta
\le\Delta
\le\frac{\lambda}{4},
\end{align*}
where the last inequality follows from $\lambda=\alpha\Delta$ and $\alpha\ge32$.
Furthermore, let $\bar\beta^{(\lambda)}=(I-\lambda\bar\Sigma_{\lambda}^{-1})\beta^*$. Then, on events $E_{\Delta}$ and $E_{\zeta}'$, by Proposition~\ref{prop:betahatstability}, we have
\begin{align*}
\|\hat\beta^{(\lambda)}-\bar\beta^{(\lambda)}\|_2
&\le\sqrt{\frac{2(\Delta+\gamma)^2\|\beta^*\|_2^2}{\lambda^2}+\frac{10m\sigma^2\zeta}{n\lambda}} \\
&\le\sqrt{\frac{2b^2(\Delta+\gamma)^2}{\lambda^2}+\frac{30\log(144/\delta)}{n\lambda}} \\
&\le\frac{2\sqrt{2}b}{\alpha}+\sqrt{\frac{30\log(144/\delta)}{n\lambda}} \\
&\le\frac{4b}{\alpha}+2\sqrt{\frac{\log(144/\delta)}{\alpha\sqrt{n}}},
\end{align*}
so the bound on $\|\hat\beta^{(\lambda)}-\hat\beta^{(\lambda)\prime}\|_2$ follows by the triangle inequality. We also note that our assumptions on $n$ and $\alpha$ imply that
\begin{align*}
\|\hat\beta^{(\lambda)}-\bar\beta^{(\lambda)}\|_2\le\frac{b}{4},
\end{align*}
which will be useful later. Next, we bound the expected policy improvement. By Proposition~\ref{prop:biasstability},
\begin{align*}
|\mathcal{P}(\hat\beta^{(\lambda)};\hat\beta^{(\lambda)\prime})-\mathcal{P}(\hat\beta^{(\lambda)};\bar\beta^{(\lambda)})|
\le\phi_{\text{max}}\|\hat\beta^{(\lambda)\prime}-\bar\beta^{(\lambda)}\|_2
\le\frac{12b}{\alpha}+6\sqrt{\frac{\log(144/\delta)}{\alpha\sqrt{n}}}.
\end{align*}
To bound $|\mathcal{P}(\hat\beta^{(\lambda)};\bar\beta^{(\lambda)})-\mathcal{P}(\bar\beta^{(\lambda)};\bar\beta^{(\lambda)})|$, note that
\begin{align*}
\bar\Sigma_{\lambda}^{-1}
=\begin{bmatrix}
\frac{1}{1+\lambda} & 0 & 0 \\
0 & \frac{1+\lambda}{\lambda(2+\lambda)} & -\frac{1}{\lambda(2+\lambda)} \\
0 & -\frac{1}{\lambda(2+\lambda)} & \frac{1+\lambda}{\lambda(2+\lambda)}
\end{bmatrix},
\end{align*}
so
\begin{align*}
\bar\beta^{(\lambda)}
=(I-\lambda\bar\Sigma_{\lambda}^{-1})\beta^*
&=\begin{bmatrix}
1-\frac{\lambda}{1+\lambda} & 0 & 0 \\
0 & 1-\frac{1+\lambda}{2+\lambda} & \frac{1}{2+\lambda} \\
0 & \frac{1}{2+\lambda} & 1-\frac{1+\lambda}{2+\lambda}
\end{bmatrix}
\begin{bmatrix}
0 \\
b \\
0
\end{bmatrix}
=\begin{bmatrix}
0 \\
\frac{b}{2+\lambda} \\
\frac{b}{2+\lambda}
\end{bmatrix}.
\end{align*}
Thus, letting $\hat\xi=\hat\beta^{(\lambda)}-\bar\beta^{(\lambda)}$, we have
\begin{align*}
\pi^{\bar\beta^{(\lambda)}}(x)&=\mathbbm{1}(x\ge0) \\
\pi^{\hat\beta^{(\lambda)}}(x)&=\mathbbm{1}(x\ge\hat\tau)
\qquad\text{where}\qquad
\hat\tau=-\frac{\hat\xi_1}{b/(2+\lambda)+\hat\xi_3}.
\end{align*}
By our above bound $\|\hat\xi\|_2=\|\hat\beta^{(\lambda)}-\bar\beta^{(\lambda)}\|_2\le b/4$, we have $|\hat\xi_3|\le\|\hat\xi\|_2\le b/4$. Since $\lambda\le1$, we have
\begin{align*}
\frac{b}{2+\lambda}+\hat\xi_3\ge\frac{b}{3}-\frac{b}{4}=\frac{b}{12}>0.
\end{align*}
Thus, we have
\begin{align*}
\hat\tau^2
\le\frac{144\|\hat\xi\|_2^2}{b^2}.
\end{align*}
Thus, we have
\begin{align*}
\mathcal{P}(\bar\beta^{(\lambda)};\bar\beta^{(\lambda)})-\mathcal{P}(\hat\beta^{(\lambda)};\bar\beta^{(\lambda)})
&=\frac{b}{2+\lambda}\cdot\mathbb{E}_x\left[x\left(\mathbbm{1}(x\ge0)-\mathbbm{1}(x\ge\hat\tau)\right)\right] \\
&=\frac{b}{(2+\lambda)\sqrt{2\pi}}(1-e^{-\hat\tau^2/2}) \\
&\le\frac{b\hat\tau^2}{2(2+\lambda)\sqrt{2\pi}} \\
&\le\frac{15\|\hat\xi\|_2^2}{b} \\
&\le\frac{480b}{\alpha^2}+\frac{120\log(144/\delta)}{\alpha b\sqrt{n}},
\end{align*}
where the last line follows from the inequality $(a+b)^2\le2a^2+2b^2$. These bounds hold for each of the four models; thus, by repeated application of the triangle inequality, we have
\begin{align*}
|\mathcal{P}(\hat\beta^{(\lambda)};\hat\beta^{(\lambda)\prime})-\mathcal{P}(\hat\beta^{(\lambda)\prime\prime};\hat\beta^{(\lambda)\prime\prime\prime})|
&\le\frac{24b}{\alpha}+12\sqrt{\frac{\log(144/\delta)}{\alpha\sqrt{n}}}+\frac{960b}{\alpha^2}+\frac{240\log(144/\delta)}{\alpha b\sqrt{n}}.
\end{align*}

\paragraph{Winner's curse.}

Finally, we have $\pi^{\bar\beta^{(\lambda)}}(x)=\mathbbm{1}(x\ge0)$, so
\begin{align*}
\mathcal{R}(\bar\beta^{(\lambda)};\bar\beta^{(\lambda)})
&=\mathbb{E}_x\left[\frac{bx}{2+\lambda}+\frac{bx\mathbbm{1}(x\ge0)}{2+\lambda}\right]-\mathbb{E}_x[bx] \\
&=\frac{b}{2+\lambda}\cdot\mathbb{E}_x[x\mathbbm{1}(x\ge0)] \\
&=\frac{b}{(2+\lambda)\sqrt{2\pi}} \\
&\ge\frac{b}{3\sqrt{2\pi}},
\end{align*}
where the last line uses our assumption that $\lambda\le1$. By our choice of $\beta^*$, for all $\beta\in\mathbb{R}^m$, we have $\mathcal{P}(\beta;\beta^*)=0$, so $\mathcal{R}(\beta;\beta')=\mathcal{P}(\beta;\beta')$. Thus, we have
\begin{align*}
|\mathcal{R}(\hat\beta^{(\lambda)};\hat\beta^{(\lambda)\prime})-\mathcal{R}(\bar\beta^{(\lambda)};\bar\beta^{(\lambda)})|
&=|\mathcal{P}(\hat\beta^{(\lambda)};\hat\beta^{(\lambda)\prime})-\mathcal{P}(\bar\beta^{(\lambda)};\bar\beta^{(\lambda)})| \\
&\le\frac{12b}{\alpha}+6\sqrt{\frac{\log(144/\delta)}{\alpha\sqrt{n}}}+\frac{480b}{\alpha^2}+\frac{120\log(144/\delta)}{\alpha b\sqrt{n}}.
\end{align*}
The claim follows. \qed

\section{Details of Refugee Matching Case Study}

This section provides the implementation details for our refugee matching simulation environment. Wherever possible, we calibrated parameters, sample sizes, and marginal distributions using those reported for the US site in the Supplemental Material of~\cite{bansak2018improving}.

\subsection{Additional Details on~\cite{bansak2018improving} and~\cite{ahani2021placement}}
\label{sec:add_deets}

Recent work proposes to use historical data to design algorithmic assignment rules that improve short-run employment relative to existing practice.

Using U.S. and Swiss data,~\cite{bansak2018improving} train machine learning models to predict employment for each refugee--location pair using features such as country of origin, language, age, gender, education, and placement restrictions. In their main specification, they fit separate prediction models by location and select gradient-boosted trees based on out-of-sample classification accuracy and calibration. They then:

\begin{enumerate}
    \item estimate predicted employment probabilities for each refugee at each feasible location;
    \item aggregate individual predictions to a family-level score (for example, the predicted probability that at least one family member is employed at a location); and
    \item solve a linear assignment problem to maximize the average family-level score, subject to capacity and practical constraints.
\end{enumerate}

To evaluate the algorithm, they run backtests: train on earlier arrivals, treat later arrivals with no placement restrictions as a test set, and compare predicted employment under the algorithmic assignment to the observed employment rate under the historical assignment. Reported gains are large, between 40--75\% improvements in predicted employment relative to existing procedures.

Similarly,~\cite{ahani2021placement} develop Annie MOORE, an integrated machine learning and integer-optimization tool for a U.S.\ resettlement agency. They focus on ``free cases''---families without U.S. ties who can be assigned flexibly across affiliates. Using administrative data on demographics, household structure, origin, language, health indicators, and macroeconomic conditions, they train predictive models of employment for refugee-affiliate pairs.

They compare pooled logistic regression, affiliate-specific logistic regression, LASSO-logit with hand-specified interactions, and gradient-boosted trees. In their test set, LASSO and gradient-boosted trees achieve substantially lower misclassification error and higher AUC than simpler models; LASSO is chosen as the primary model because it combines good discrimination with reasonable calibration.

Given these predictions, they:

\begin{enumerate}
    \item treat the predicted employment probabilities as quality scores for each refugee--affiliate pair;
    \item solve a mixed-integer program that allocates refugees to affiliates, maximizing total predicted employment subject to capacity and operational constraints; and
    \item evaluate gains by comparing the sum of predicted employment under the optimized allocation to the observed employment under historical assignment, using the same LASSO model to impute counterfactual outcomes.
\end{enumerate}

These backtests suggest improvements in predicted employment on the order of 20--40\%. In an e-companion, they also conduct a bootstrap analysis: they draw many bootstrap samples of the training data, re-estimate the LASSO model on each, and evaluate the \emph{same} optimized allocation under each bootstrapped model. The resulting distribution of estimated gains is tight and positive, which they interpret as evidence of robustness.

\subsection{Covariate Marginals and Location Assignment Probabilities}
\label{apx:marginals}

\begin{table}
\centering \small
\begin{tabular}{c|c}
Covariate & Probability \\
\hline
\hline
$\text{Age}\sim \mathrm{Unif}(18,30)$& 0.44\\
$\text{Age}\sim \mathrm{Unif}(30,40)$& 0.28\\
$\text{Age}\sim \mathrm{Unif}(40,50)$& 0.16\\
$\text{Age}\sim \mathrm{Unif}(50,60)$& 0.12\\
\hline
Gender = Male & 0.53 \\
\hline
Education = None& 0.18\\
Education = Little& 0.39\\
Education = Secondary& 0.21\\
Education = Advanced& 0.10\\
Education = University& 0.12\\
\hline
Speaks English = True& 0.43 \\
\hline
Case Restriction = Free&  0.28\\
\hline
Origin = Burma& 0.23\\
Origin = Iraq& 0.2\\
Origin = Bhutan& 0.13\\
Origin = Somalia& 0.11\\
Origin = Afghanistan& 0.07\\
Origin = Dem. Rep. of Congo& 26/900\\
Origin = Iran& 26/900\\
Origin = Eritrea& 26/900\\
Origin = Ukraine& 26/900\\
Origin = Syria& 26/900\\
Origin = Sudan& 26/900 \\
Origin = Ethiopia& 26/900 \\
Origin = Moldova& 26/900 \\
Origin = Other& 26/900 \\
\hline
$\text{Arrival Year}\sim \mathrm{Unif} \{2011,\ldots,2015\}$& 20/23\\
Arrival Year = 2016& 3/23\\
\hline
$\text{Arrival Month}\sim \mathrm{Unif} \{\text{Jan},\ldots,\text{Sep}\}$& 6/69\\
$\text{Arrival Month}\sim \mathrm{Unif} \{\text{Oct},\text{Nov},\text{Dec}\}$& 5/69\\
\end{tabular}
\caption{Probability of covariates being assigned (independently) to each refugee, following the marginals reported in the Supplemental Material of~\cite{bansak2018improving}. For binary covariates only one option is reported.}
\label{tab:covar_margin}
\end{table}

\newpage

\begin{figure}[H]
\centering
\includegraphics[width=0.67\linewidth]{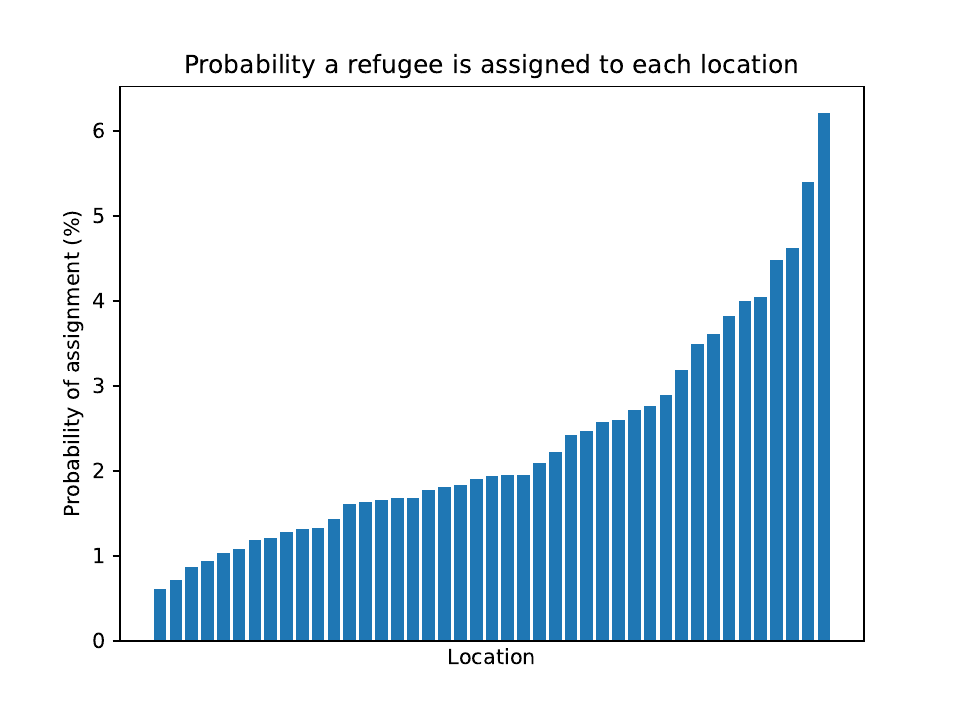}
\caption{Probability of each location being assigned to each refugee sample, sorted in ascending order. This distribution follows the empirical distribution reported in~\cite{bansak2018improving}.}
\label{fig:p_location}
\end{figure}

\subsection{Employment Rates by Location}
\label{apx:emp_loc}

\begin{figure}[H]
\centering
\includegraphics[width=0.67\linewidth]{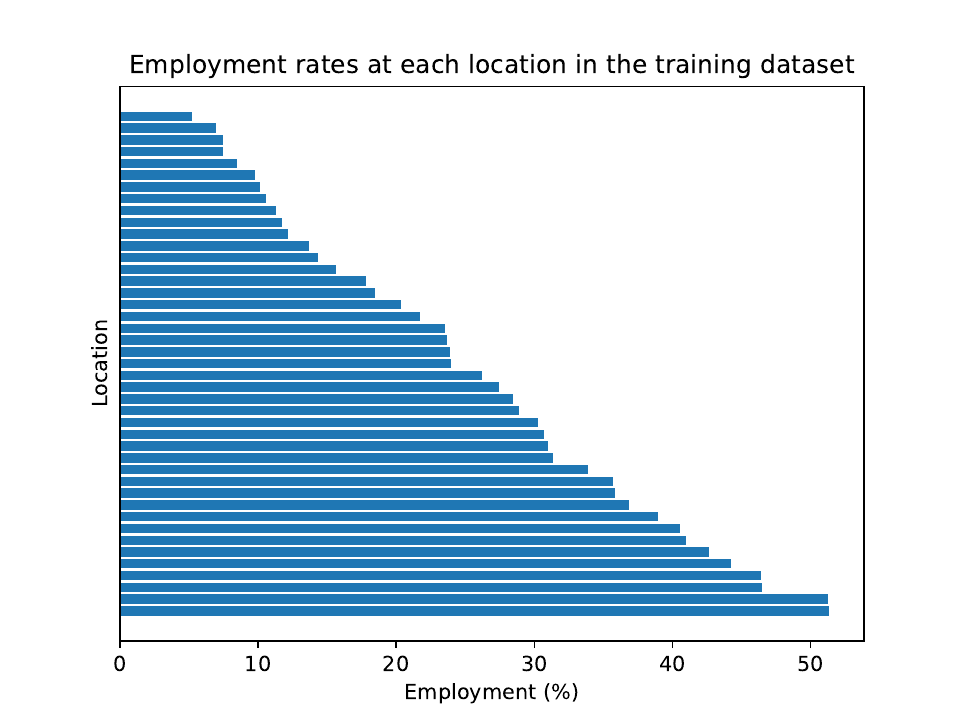}
\caption{The distribution of employment rates across location have been tuned to resemble those in~\cite{bansak2018improving} to maximize similarities between the simulation environment and real world data.}
\label{fig:emp_location}
\end{figure}

\subsection{Prediction Model Calibration and ROC Curves}
\label{apx:cal_ROC}

The calibration and ROC curves reflect that each prediction model is well-calibrated and makes fairly accurate predictions.

\begin{figure}[H]
\centering
\begin{subfigure}[b]{0.31\textwidth}
\centering
\includegraphics[width=\textwidth]{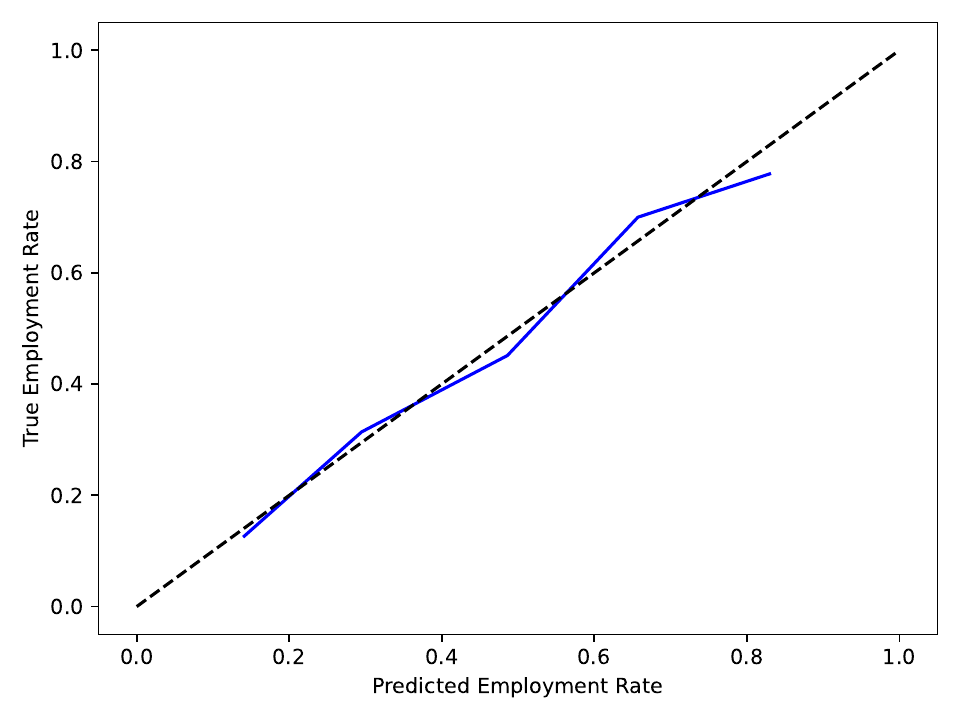}
\caption{LASSO}
\label{fig:cal_lasso}
\end{subfigure}
\hfill
\begin{subfigure}[b]{0.31\textwidth}
\centering
\includegraphics[width=\textwidth]{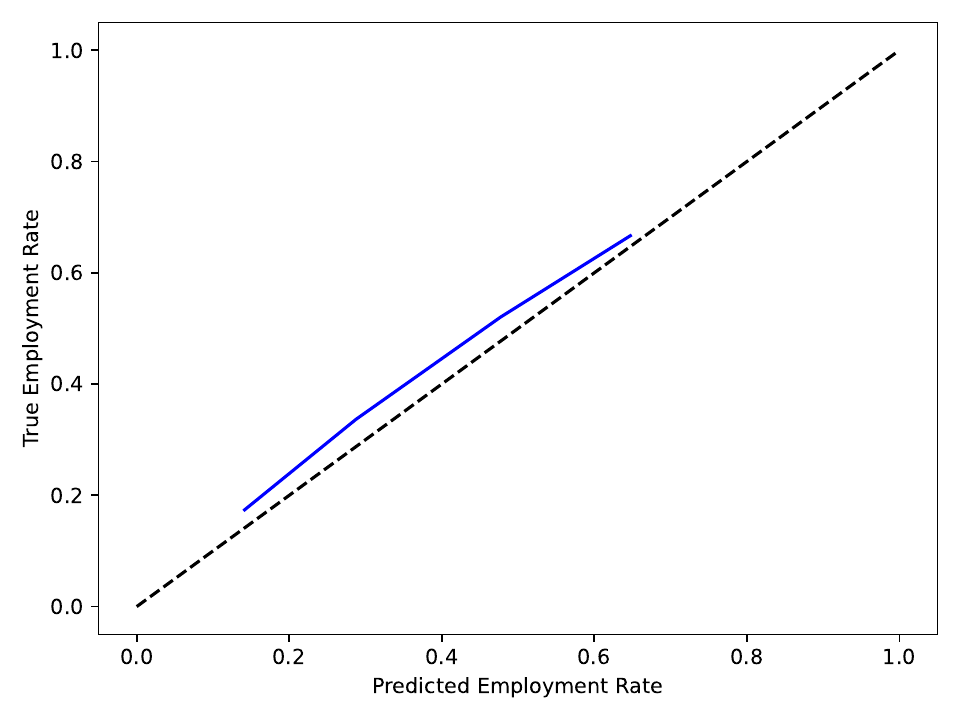}
\caption{Honest RF}
\label{fig:cal_hf}
\end{subfigure}
\hfill
\begin{subfigure}[b]{0.31\textwidth}
\centering
\includegraphics[width=\textwidth]{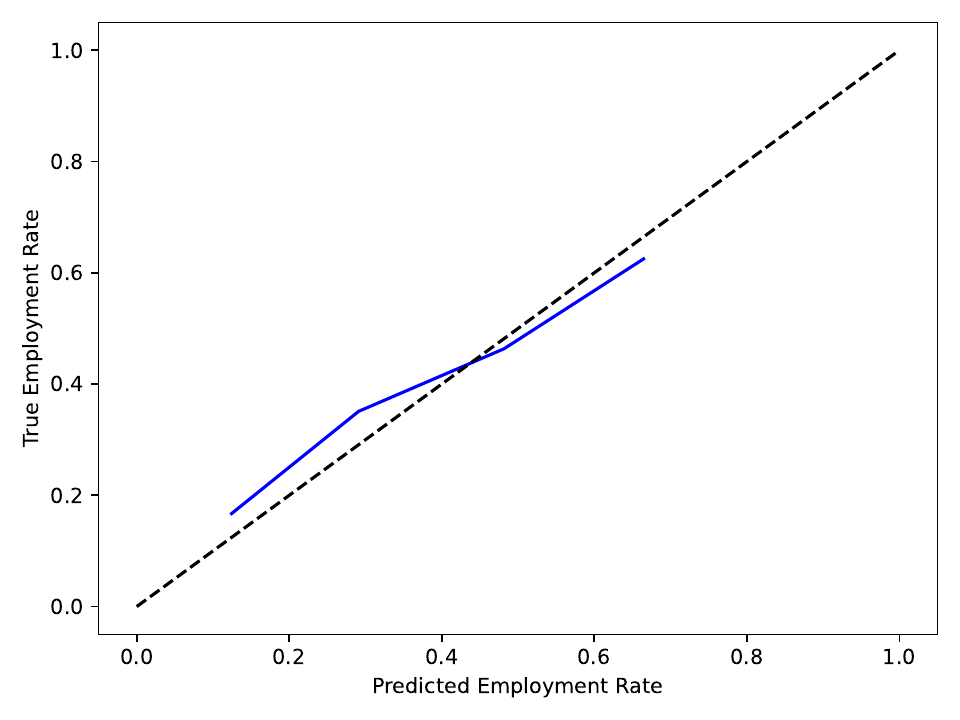}
\caption{GBM}
\label{fig:cal_gbm}
\end{subfigure}
\caption{Calibration curves for each prediction model.}
\label{fig:cal}
\end{figure}

\begin{figure}[H]
\centering
\begin{subfigure}[b]{0.31\textwidth}
\centering
\includegraphics[width=\textwidth]{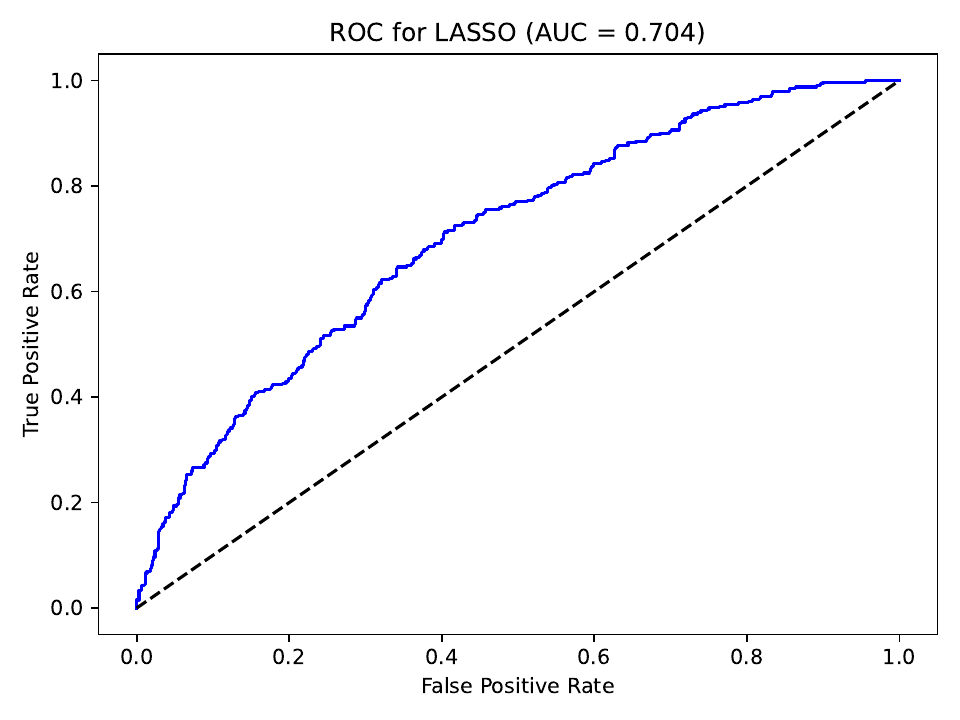}
\caption{LASSO (AUC = 0.704)}
\label{fig:roc_lasso}
\end{subfigure}
\hfill
\begin{subfigure}[b]{0.31\textwidth}
\centering
\includegraphics[width=\textwidth]{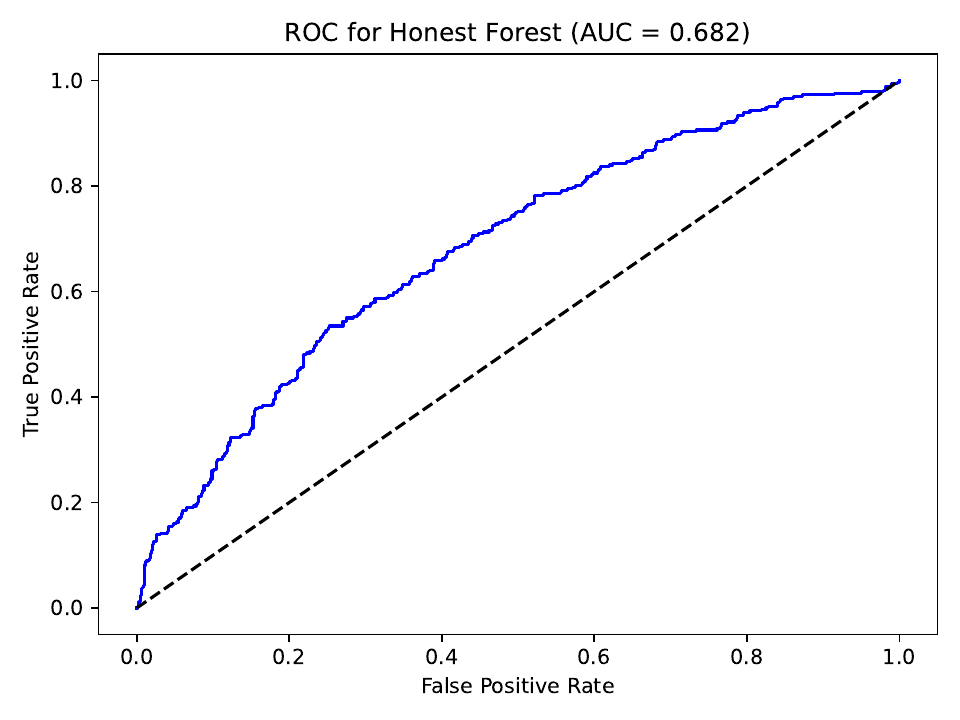}
\caption{Honest RF (AUC = 0.682)}
\label{fig:roc_hf}
\end{subfigure}
\hfill
\begin{subfigure}[b]{0.31\textwidth}
\centering
\includegraphics[width=\textwidth]{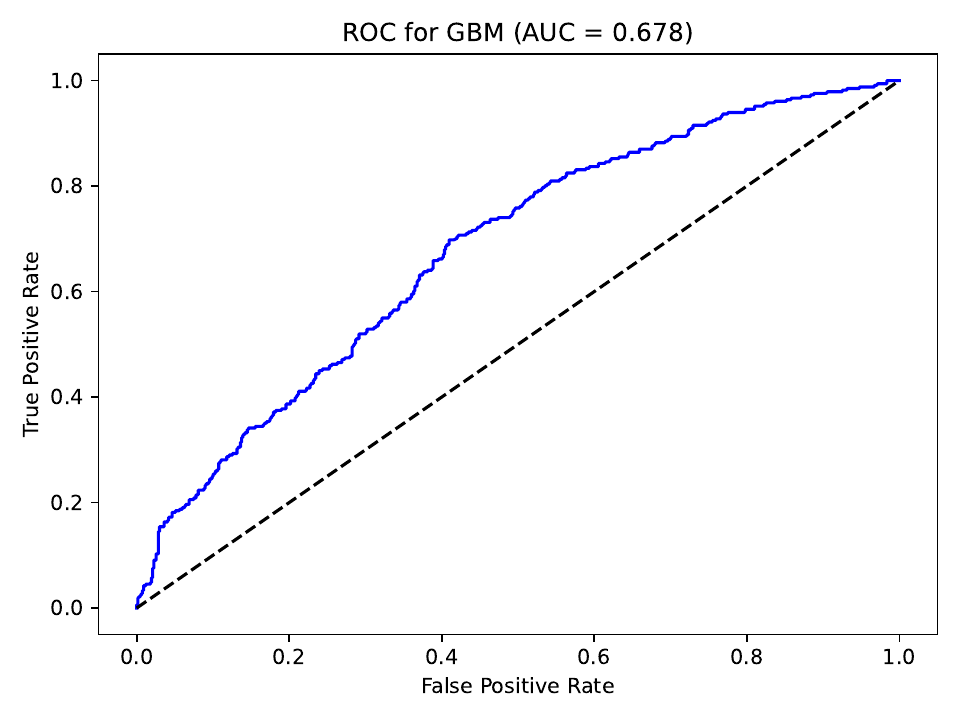}
\caption{GBM (AUC = 0.678)}
\label{fig:roc_gbm}
\end{subfigure}
\caption{ROC curves detailing the performance of each prediction model on the test dataset.}
\label{fig:ROC_Curves}
\end{figure}

\end{document}